\documentclass[10pt,journal,compsoc]{IEEEtran}

\usepackage[nocompress]{cite}
\usepackage{times}
\usepackage{soul}
\usepackage{url}
\usepackage[hidelinks]{hyperref}
\usepackage[utf8]{inputenc}
\usepackage[small]{caption}   
\usepackage{graphicx}
\usepackage{amsthm}
\usepackage{booktabs}
\usepackage{algorithm}
\usepackage{algorithmic}
\newtheorem{theorem}{Theorem}
\usepackage{threeparttable}
\usepackage{tabularx}
\usepackage{enumerate}
\usepackage{bbm}
\usepackage{bm}
\usepackage{epsfig, epstopdf}
\usepackage{multirow}
\usepackage{xcolor}

\usepackage{flushend}
\usepackage{subfigure}
\usepackage{amsmath,amsfonts,amssymb}
\usepackage{lineno}
\usepackage{siunitx}
\usepackage{CJKutf8}
\usepackage{wrapfig}
\usepackage{indentfirst}
\usepackage[T1]{fontenc}
\usepackage{ragged2e}
\renewcommand{\raggedright}{\leftskip=0pt \rightskip=0pt plus 0cm}

\ifCLASSOPTIONcompsoc
  \usepackage[nocompress]{cite}
\else
  \usepackage{cite}
\fi
\ifCLASSINFOpdf
\else
\fi
\begin{document}
%
\title{Self-Supervised Discriminative Feature Learning for Deep Multi-View Clustering}
%
%
%
%

\author{Jie Xu, Yazhou Ren,~\IEEEmembership{Member,~IEEE}, Huayi Tang, Zhimeng Yang, Lili Pan, Yang Yang,~\IEEEmembership{Senior Member,~IEEE}, Xiaorong Pu, Philip S. Yu,~\IEEEmembership{Fellow,~IEEE}, Lifang He,~\IEEEmembership{Member,~IEEE}
\IEEEcompsocitemizethanks{
\IEEEcompsocthanksitem This work was supported in part by National Natural Science Foundation of China (No. 61806043), Sichuan Science and Technology Program (Nos. 2021YFS0172, 2020YFS0119, and 2020YFG0288), and Special Science Foundation of Quzhou (No. 2020D013). Philip S. Yu was supported in part by NSF under grants III-1763325, III-1909323,  III-2106758, and SaTC-1930941. Lifang He was supported by Lehigh's accelerator grant S00010293.
\IEEEcompsocthanksitem Jie Xu (jiexuwork@outlook.com), Yazhou Ren (yazhou.ren@uestc.edu.cn), Huayi Tang, Zhimeng Yang, Yang Yang, and Xiaorong Pu are with the School of Computer Science and Engineering, University of Electronic Science and Technology of China, Chengdu 611731, China.
\IEEEcompsocthanksitem Lili Pan is with the School of Information and Communication Engineering, University of Electronic Science and Technology of China, Chengdu 611731, China.
\IEEEcompsocthanksitem Philip S. Yu is with the Department of Computer Science, University of Illinois at Chicago, Chicago, IL 60607, USA.
\IEEEcompsocthanksitem Lifang He is with the Department of Computer Science and Engineering, Lehigh University, Bethlehem, PA 18015, USA.
}
\thanks{
~ \\
(Corresponding author: Yazhou Ren.)}
}

\IEEEtitleabstractindextext{
\begin{abstract}
Multi-view clustering is an important research topic due to its capability to utilize complementary information from multiple views. However, there are few methods to consider the negative impact caused by certain views with unclear clustering structures, resulting in poor multi-view clustering performance. To address this drawback, we propose \underline{s}elf-supervised discriminative feature learning for \underline{d}eep \underline{m}ulti-\underline{v}iew \underline{c}lustering (SDMVC). Concretely, deep autoencoders are applied to learn embedded features for each view independently. To leverage the multi-view complementary information, we concatenate all views' embedded features to form the global features, which can overcome the negative impact of some views' unclear clustering structures. In a self-supervised manner, pseudo-labels are obtained to build a unified target distribution to perform multi-view discriminative feature learning. During this process, global discriminative information can be mined to supervise all views to learn more discriminative features, which in turn are used to update the target distribution. Besides, this unified target distribution can make SDMVC learn consistent cluster assignments, which accomplishes the clustering consistency of multiple views while preserving their features' diversity. Experiments on various types of multi-view datasets show that SDMVC outperforms 14 competitors including classic and state-of-the-art methods. The code is available at \url{https://github.com/SubmissionsIn/SDMVC}.
\end{abstract}

}

\maketitle

\IEEEdisplaynontitleabstractindextext

%
\IEEEpeerreviewmaketitle

\IEEEraisesectionheading{\section{Introduction}\label{sec:introduction}}

%
%
%
%
\IEEEPARstart{C}{lustering} analysis is a fundamental task in unsupervised learning and has been applied in a wide range of data mining fields, such as knowledge engineering \cite{mineau1995automatic}, information retrieval \cite{li2020weakly}, and pattern recognition \cite{tang2015rgb}, etc.
However, traditional clustering methods are generally inapplicable in the scenarios where the real-world data are collected from multiple views or multiple modalities as known as multi-view data or multimodal data, e.g., (1) multiple mappings of one object, (2) visual feature + textual feature, and (3) scale-invariant feature transform (SIFT) + local binary pattern (LBP).
Therefore, multi-view clustering (MVC) becomes a hot research topic that can leverage complementary information and comprehensive characteristics hidden in multi-view data.

\begin{figure}[!t]
\centering
\subfigure[]{\includegraphics[height=1.362in, width=1.7in]{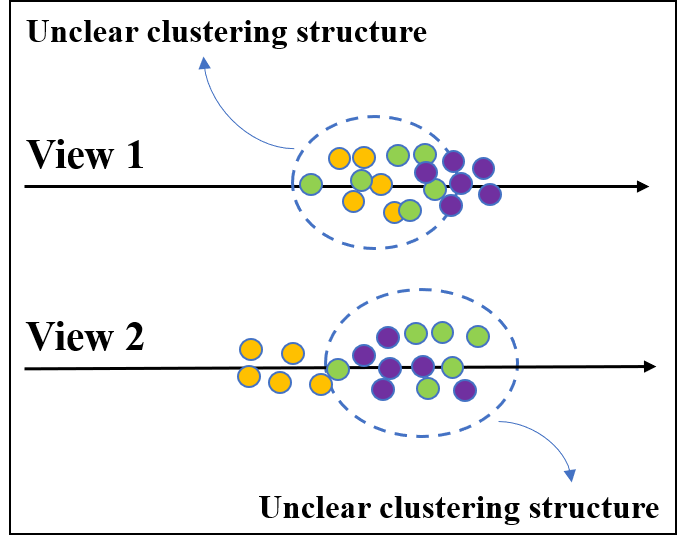}}   
\subfigure[]{\includegraphics[height=1.362in, width=1.7in]{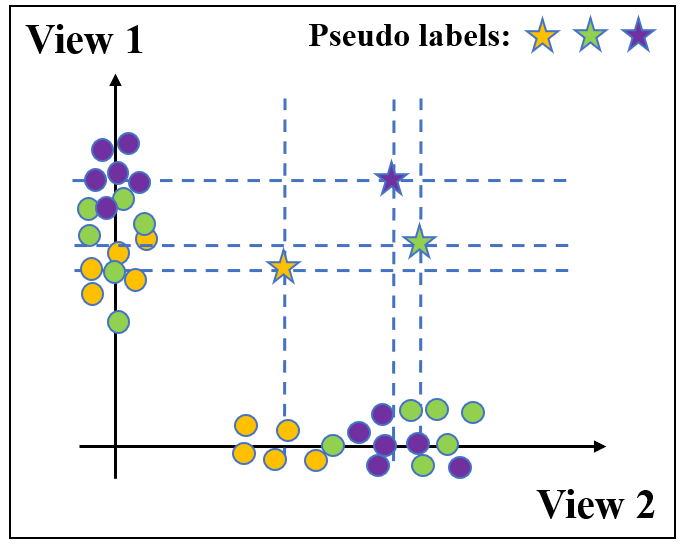}}
\subfigure[]{\includegraphics[height=1.362in, width=1.7in]{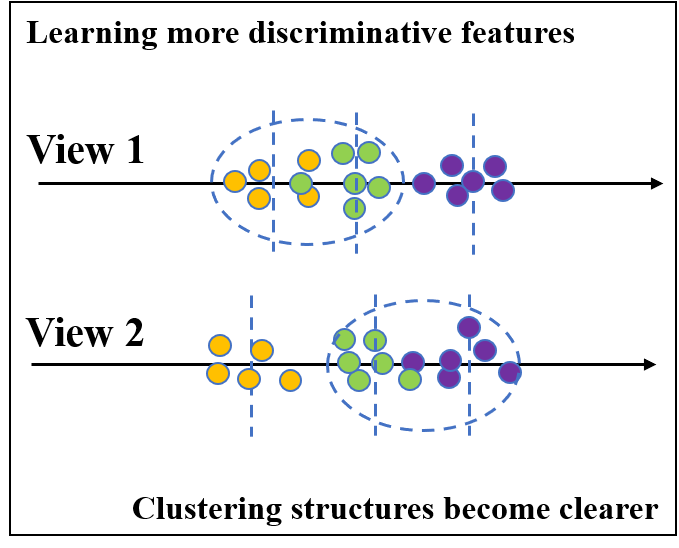}}
\subfigure[]{\includegraphics[height=1.362in, width=1.7in]{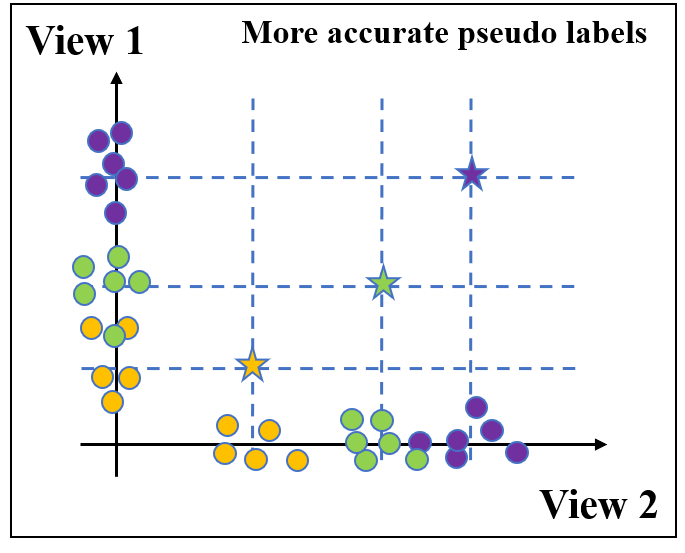}}
\caption{An illustrative example of our motivation.}
\label{fig:motivation}
\end{figure}

MVC can be roughly classified into four categories: 
(1) Numerous MVC methods are based on subspace clustering \cite{zhang2018generalized,Ming2018Multiview,2020Robust,2021Rank}, where the shared representation of multiple views and similarity metric matrix are explored. 
(2) Some MVC methods \cite{liu2013multi,wang2018multiview,yang2020uniform} apply non-negative matrix factorization to decompose each view into a low-rank matrix for clustering. 
(3) In the third MVC category \cite{nie2017self,Zhan8052206,wang2019gmc,xiao2019knowledge,peng2019comic,wang2019parameter}, graph-based structure information is integrated to mine clusters among multiple views.
(4) Deep learning techniques are also applied to MVC in recent years, such as \cite{li2019deep,zhu2019multi,xie2019multiview,xu2021deep,wang2020deep,LiuAAAI2022}. Deep MVC aims to obtain better performance with the feature representation capability of deep models.
More details can be found in \cite{yang2018multi,wang2021Survey} which provide a comprehensive survey about multi-view clustering.

\begin{figure*}[!th]
\centering
\subfigure{\includegraphics[width=7.1in]{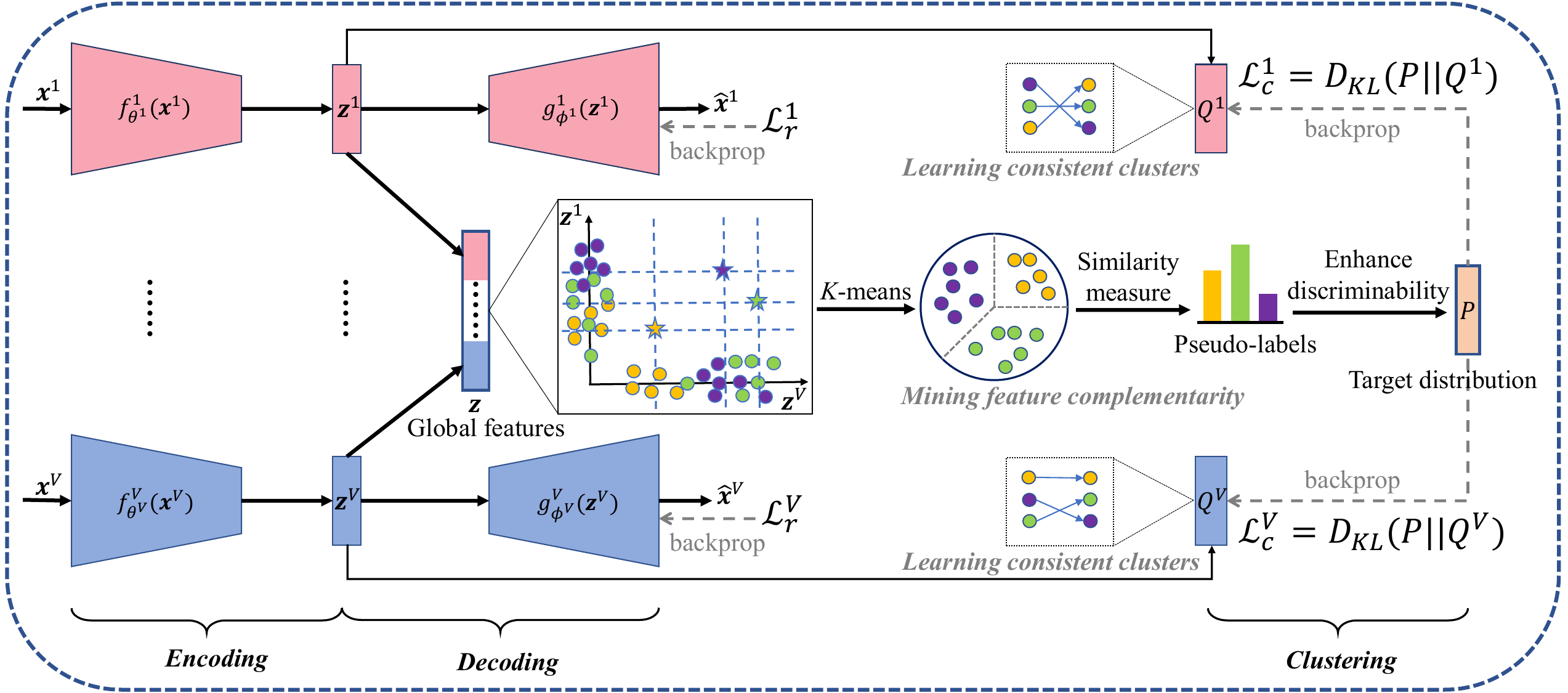}}
\caption{The framework of SDMVC. Each view contains an autoencoder and a clustering layer. For the $v$-th view, $f^v_{\theta^v}$ denotes the encoder and $g^v_{\phi^v}$ denotes the decoder.
$\pmb{z}^v$ is the embedded features learned by autoencoder and $Q^v$ is the cluster assignment distribution output by the clustering layer. The unified target distribution $P$ is iteratively updated to conduct multi-view discriminative feature learning. In this way, the feature complementarity and the consistent multi-view clustering are achieved.}
\label{fig:framework}
\end{figure*}

Although existing MVC methods have achieved significant progress in the past decade, their performance is still limited in the following issues.
Firstly, some MVC methods have to depend on many hyper-parameters, which might be sensitive to different datasets.
Since the real-world clustering tasks usually have no label information to help select the setting of hyper-parameters, the methods with less or insensitive hyper-parameters are needed.
In addition, previous works usually conduct the consistency of multiple views directly in the embedded feature space, which reduces the capability of features to preserve the diversity information among multiple views. The consequence is, if certain views' clustering structures are highly fuzzy, the effect of the views with clear clustering structures will be limited, which results in poor clustering performance in all views.

To overcome the aforementioned issues, we propose self-supervised discriminative feature learning for deep multi-view clustering (SDMVC).
Our motivation comes from the observation illustrated in Fig.~\ref{fig:motivation}: (a) The discriminability of multiple views is different due to their diversity (for example, red motorcycles and red bicycles are almost indistinguishable from a color view, but they are clearly distinguishable from a semantic view). So, the discriminative degree of different views' features are different and higher discriminative features are more conducive to clustering (i.e., higher discriminative features have clearer clustering structures).
(b) When the features are concatenated, the clustering structures with high discriminability play a major role in dividing the global feature space, which can produce pseudo-labels of high confidence and overcome the negative effects of unclear clustering structures. 
(c)--(d) The pseudo-labels can be used to lead all views to learn more discriminative features, which further produce clearer clustering structures and more accurate pseudo-labels.

Specially, the proposed \emph{self-supervised multi-view discriminative feature learning} framework is shown in Fig. \ref{fig:framework}.
For each view, a corresponding autoencoder is employed to learn low-dimensional embedded features $\pmb{z}^v$ by optimizing the reconstruction loss $\mathcal{L}_r^v$. 
The discriminability of different views' embedded features is different.
To leverage the global discriminative information, we concatenate all embedded features to build global features, pseudo-labels, and a unified target distribution, in sequence.
During this process, the feature complementarity can be mined and the negative effects of the views with unclear clustering structures can be overcome.
Then, the feature learning is performed by optimizing the KL divergence $\mathcal{L}_c^v$ between the unified target distribution $P$ and each view's cluster assignments $Q^v$.
In this way, global discriminative information can be used to lead all views to learn more discriminative features, which in turn help to obtain complementary information and a more accurate target distribution. We further show that SDMVC is able to learn the consistent cluster assignments by optimizing the multi-view clustering loss, which leverages the proposed unified target distribution to explore the clustering consistency across all views.
Additionally, the proposed independent training for each view allows all views' features to preserve their diversity so as to be robust to the views with different discriminability.

In summary, the contributions of this paper include:
\begin{itemize}
\item We propose a deep MVC method with a novel self-supervised multi-view discriminative feature learning framework, which can leverage the global discriminative information contained in all views' embedded features to perform multi-view clustering.
\item We theoretically analyze the generalization bound of the proposed loss function, which demonstrates that the expected clustering risk of our method is bounded by the empirical risk on training samples with high probability.
\item From the perspective of cluster assignments rather than features, the proposed framework simultaneously encourages the clustering consistency of multiple views while preserving their features' diversity. In addition, it can overcome the negative impact on multi-view clustering caused by some views with unclear clustering structures.
\item The proposed method has less hyper-parameters and its complexity is linear to the data size. Experiments on different types of datasets demonstrate its superior performance over recent state-of-the-art methods.
\end{itemize}

\section{Related Work}
\label{sec:related}
\subsection{Deep Embedded Clustering}
Deep autoencoder, one of the popular deep models, has good embedded feature representation capability and its computing complexity is linear to data size.
Recently, deep autoencoder based clustering has been extensively studied.
For example, one of the most popular works is deep embedded clustering (DEC) \cite{xie2016unsupervised}, which jointly learns the cluster assignments and embedded features of autoencoders. The improved deep embedded clustering (IDEC) \cite{guo2017improved} introduced a trade-off between clustering and reconstruction to prevent the collapse of embedded space. \cite{ghasedi2017deep} proposed another variant of DEC which stacks multinomial logistic regression on top of a convolutional autoencoder. More researches based on deep embedded clustering can be found in \cite{guo2018deep,ren2019semi,2020Joint}.
Although many works based on deep autoencoder have achieved impressive performance \cite{xie2016unsupervised,ghasedi2017deep,yang2017towards,ren2020deep}, these methods focus on single-view clustering tasks and are unable to handle multi-view clustering tasks.
In this paper, we propose a novel deep multi-view clustering framework where discriminative embedded features are learned in a self-supervised manner.

\subsection{Multi-View Clustering}
Subspace-based multi-view clustering is widely discussed, which assumes that the data of multiple views come from a same latent space. In \cite{li2019reciprocal}, the authors explored self-representation layers to reconstruct view-specific subspace hierarchically and encoding layers to make cross-view subspace more consensus. Recently, in \cite{xie2019multiview} and \cite{zheng2020feature}, multi-view data were transformed into joint representations to perform multi-view clustering. Another work \cite{chen2020multi} learned multi-view embeddings and jointly mined common structure and cluster assignments in the learned latent space. 
Some multi-view clustering methods were implemented by using non-negative matrix factorization techniques, e.g., by which Liu \emph{et al.} \cite{liu2013multi} explored a common latent factor among multiple views. A deep structure \cite{zhao2017multi} was built to seek common features with more consistent information. 
Some methods \cite{nie2017self,wang2019gmc,huang2020auto,tang2020cgd,li2021consensus} exploited graph-based models for multi-view clustering. For example, a multi-graph laplacian regularized low-rank representation was proposed in \cite{wang2016iterative} for multi-view spectral clustering. Peng \emph{et al.} \cite{peng2019comic} mined geometry and cluster alignment consistency in projection space by connection graph. In \cite{fan2020one2multi}, graph autoencoders were also introduced to learn multi-view representations.

Some works are based on other learning approaches. For instance, Zhang \emph{et al.} \cite{zhang2018binary} incorporated discrete representation learning and binary clustering into a unified framework. In \cite{ren2020self}, self-paced learning was applied to prevent being stuck in local optima.
\cite{li2018deep} proposed a novel deep collaborative embedding method, which could effectively explore the multi-view information and achieve the promising performance for social image understanding.
In addition, deep learning is an attention-getting trend.
For example, 
Zhu \emph{et al.} \cite{zhu2019multi} trained deep autoencoders to obtain self-representation and leveraged the diversity and universality regularization to abstract higher-order relation among views.
Adversarial learning based deep MVC was proposed in \cite{li2019deep,Zhou2020end}, which aim to learn the intrinsic embedded structure of multi-view data.
Self-supervised learning is the recent hot topic of the community. The framework proposed in \cite{sun2019self} combined a self-supervised paradigm with multi-view clustering. However, it belongs to subspace clustering and depends on the eigenvalue decomposition, which causing cubic complexity to the data size.

For multi-view clustering, it is usual to obtain consistent predictions of multiple views for the same examples.
Wang \emph{et al.} \cite{wang2019multi} maximized the alignment between weighted view-specific partition and consensus partition.
The referred view was introduced in \cite{xu2021deep} to explicitly train all views to achieve consistent predictions.
\cite{liu2014partially} proposed a novel semi-supervised multi-view learning algorithm to jointly explore the complementary and consistent information across multiple views.
The latest work \cite{trostenMVC} introduced two deep MVC models when aligning the distributions of multi-view representations.
Unlike previous methods, our approach achieves consistent multi-view clustering predictions by establishing a unified target distribution.
Besides, the self-supervised multi-view discriminative feature learning framework is first proposed in our work.

\section{The Proposed Method}\label{sec:proposed method}
\textbf{Problem Statement}. Given a multi-view dataset $\{\pmb{x}^1_i \in \mathbb{R}^{D_1}, \pmb{x}^2_i \in \mathbb{R}^{D_2}, \dots, \pmb{x}^V_i \in \mathbb{R}^{D_V} \}_{i=1}^N $, $V$ is the total number of views, $N$ is the number of examples, and $D_1, D_2, \dots, D_V$ are the dimensionality of views. Multi-view clustering aims to partition the examples into $K$ clusters.
\subsection{Self-Supervised Multi-View Discriminative Feature Learning}
Generally, multiple views of an example are different in dimension or input form. In order to obtain the features that are convenient to clustering, we use deep autoencoders to perform feature representation learning for each view. Specifically, let $f_{\theta^v}^v$ and $g_{\phi^v}^v$ represent the encoder and decoder of the $v$-th view, respectively. The parameters $\theta^v$ and $\phi^v$ implement the nonlinear mapping of the $v$-th autoencoder, whose encoder part learn the low-dimensional features by
\begin{equation}\label{eq:encoder}
    \pmb{z}_i^v=f^v_{\theta^v}{(\pmb{x}_i^v)},
\end{equation}
\noindent where $\pmb{z}_i^v \in \mathbb{R}^{d_v}$ is the embedded point of $\pmb{x}_i^v$ in $d_v$-dimensional feature space. The decoder part of autoencoder reconstructs the example as $\hat{\pmb{x}}_i^v \in \mathbb{R}^{D_{v}}$ by decoding $\pmb{z}_i^v$:
\begin{equation}\label{eq:decoder}
    \hat{\pmb{x}}_i^v=g^v_{\phi^v}{(\pmb{z}_i^v)}=g^v_{\phi^v}{(f^v_{\theta^v}{(\pmb{x}_i^v)})}.
\end{equation}
For each view, the reconstruction loss between input and output is optimized, so as to transform the input of various forms into low-dimensional embedded features:
\begin{equation}\label{eq:Lr}
    \mathcal{L}_r^v= \sum_{i=1}^N
    \left\|
    {\pmb{x}_i^v-g^v_{\phi^v}{(f^v_{\theta^v}{(\pmb{x}_i^v)})}}
    \right\|
    _2^2.
\end{equation}

On the top of each view's embedded features, we construct a clustering layer $c^v_{\mu^v}$ with learnable parameters $\{\pmb{\mu}_j^v \in \mathbb{R}^{d_v}\}_{j=1}^K$. $\pmb{\mu}_j^v$ represents the $j$-th cluster centroid of the $v$-th view. Additionally, DEC \cite{xie2016unsupervised} and IDEC \cite{guo2017improved} are popular single-view deep clustering methods which apply Student's $t$-distribution~\cite{maaten2008visualizing} to generate soft cluster assignments. A similar way is applied to calculate the output $Q^v$ of the $v$-th view's clustering layer, which can be described as:
\begin{equation}\label{eq:q}   
    q_{ij}^v=c^v_{\mu^v}{(\pmb{z}_i^v)}=\frac{(1+\lVert \pmb{z}_i^v-\pmb{\mu}_j^v\rVert^2)^{-1}}
    {\sum_{j}(1+\lVert \pmb{z}_i^v-\pmb{\mu}_{j}^v\rVert^2)^{-1}},
\end{equation}
\noindent where $q_{ij}^v \in Q^v$ is the probability (soft cluster assignment) that the $i$-th example belongs to the $j$-th cluster in the $v$-th view. As shown in Fig. \ref{fig:framework}, our framework has $V$ clustering layers and $V$ autoencoders. All view's embedded features are learned independently, which is essential to learn each view's special characteristics that can provide complementary information for multi-view clustering.

Actually, the soft assignments calculated by Eq.~(\ref{eq:q}) measure the similarity between embedded features and cluster centroids, so the clustering performance depends on the discriminative degree of embedded features. The discriminability (discriminative degree) of different views is different due to clear or unclear clustering structures. Since embedded feature $\pmb{z}_i^v$ only contains the discriminative information of the $v$-th view, we start with a global perspective and then define a unified target distribution to conduct multi-view discriminative feature learning.
Concretely, to leverage the discriminative information across all views, we concatenate all embedded features 
to generate the global feature:
\begin{equation}\label{eq:gf}
    \pmb{z}_i=[\pmb{z}_i^1 ;\ \pmb{z}_i^2 ;\ \dots ;\ \pmb{z}_i^V] \in \mathbb{R}^{\sum_{v=1}^V d_v}.
\end{equation}
After that, $K$-means \cite{macqueen1967some} is applied on the global features to calculate the cluster centroids: 
\begin{equation}\label{eq:kmeans}
    \mathop{\min_{\pmb{c}_1,\pmb{c}_2,\dots,\pmb{c}_K}}{\sum_{i=1}^N\sum_{j=1}^K \left\| {\pmb{z}_i-\pmb{c}_j} \right\| ^2},
\end{equation}
\noindent where $\pmb{c}_j$ is the $j$-th cluster centroid. Letting $\pmb{c}_j^1,\pmb{c}_j^2,\dots,\pmb{c}_j^V$ denote the components of cluster centroids $\pmb{c}_j$, i.e., $\pmb{c}_j = [\pmb{c}_j^1 ;\ \pmb{c}_j^2 ;\ \dots ;\ \pmb{c}_j^V]$, we then investigate the roles of the views with clear structures (high discriminative views) on the views with unclear structures (low discriminative views). Concretely, Eq. (\ref{eq:kmeans}) can be rewritten as follows:
\begin{equation}\label{7}
\begin{aligned}
    &\mathop{\min_{\pmb{c}_1,\pmb{c}_2,\dots,\pmb{c}_K}}{\sum_{i=1}^N\sum_{j=1}^K \left\| {\pmb{z}_i-\pmb{c}_j} \right\| ^2} \\
    =& \sum_{v=1}^V \mathop{\min_{\pmb{c}_1^v,\pmb{c}_2^v,\dots,\pmb{c}_K^v}}{\sum_{i=1}^N \sum_{j=1}^K \left\| {\pmb{z}_i^v-\pmb{c}_j^v} \right\| ^2}, \\
    s.&t.~\pmb{c}_j = [\pmb{c}_j^1 ;\ \pmb{c}_j^2 ;\ \dots ;\ \pmb{c}_j^V], ~j\in \{1, 2, \dots, K\}. 
\end{aligned}
\end{equation}
Furthermore, we take two clusters in two views and a global feature $\pmb{z}_i = [\pmb{z}_i^1 ;\ \pmb{z}_i^2]$ as an example. For $\pmb{z}_i$, we have
\begin{equation}\label{8}
\begin{aligned}
    & \mathop{\min_{\pmb{c}_1^1,\pmb{c}_2^1}}{\left(\left\| {\pmb{z}_i^1-\pmb{c}_1^1} \right\| ^2 + \left\| {\pmb{z}_i^1-\pmb{c}_2^1} \right\| ^2\right)} \\
    + & \mathop{\min_{\pmb{c}_1^2,\pmb{c}_2^2}}{\left(\left\| {\pmb{z}_i^2-\pmb{c}_1^2} \right\| ^2 + \left\| {\pmb{z}_i^2-\pmb{c}_2^2} \right\| ^2\right)}.
\end{aligned}
\end{equation}
Suppose the clustering structures of the 1-st view are unclear, which makes $\pmb{z}_i^1$ have no clear cluster assignment. In the extreme, $\pmb{z}_i^1$ has equal distances to the two cluster centroids, i.e., $\left\| {\pmb{z}_i^1-\pmb{c}_1^1} \right\| ^2 = \left\| {\pmb{z}_i^1-\pmb{c}_2^1} \right\| ^2$. Meanwhile, the clustering structures of the 2-nd view are clear, which makes $\pmb{z}_i^2$ have clear cluster assignment, e.g., $\left\| {\pmb{z}_i^2-\pmb{c}_1^2} \right\| ^2 > \left\| {\pmb{z}_i^2-\pmb{c}_2^2} \right\| ^2$. Then, we have
\begin{equation}\label{9}
\begin{aligned}
    \left\| {\pmb{z}_i-\pmb{c}_1} \right\| ^2 = &\left\| {\pmb{z}_i^1-\pmb{c}_1^1} \right\| ^2 + \left\| {\pmb{z}_i^2-\pmb{c}_1^2} \right\| ^2 \\
    > & \left\| {\pmb{z}_i^1-\pmb{c}_2^1} \right\| ^2 + \left\| {\pmb{z}_i^2-\pmb{c}_2^2} \right\| ^2 = \left\| {\pmb{z}_i-\pmb{c}_2} \right\| ^2. \\
\end{aligned}
\end{equation}
This indicates that the global feature $\pmb{z}_i$ is assigned to the same cluster (i.e., $\pmb{c}_2 = [\pmb{c}_2^1;\pmb{c}_2^2]$) as the view with clear structures. 
Therefore, when applying $K$-means on the global features, the features of some high discriminative views play a major role in dividing the global feature space. This process guarantees that the pseudo soft assignments achieve the robustness to the unclear clustering structures, which allows the model to overcome the negative effects caused by the low discriminative views.

Then, we can further choose the Student's $t$-distribution as a kernel to measure the similarity between global feature $\pmb{z}_i$ and cluster centroid $\pmb{c}_j$. The similarity is used to generate pseudo soft assignment (pseudo-label) $s_{ij}$ to perform self-supervised learning, which is calculated by
\begin{equation}\label{eq:pk}        
    s_{ij}=\frac{(1+\lVert \pmb{z}_i-\pmb{c}_j\rVert^2)^{-1}}
    {\sum_{j}(1+\lVert \pmb{z}_i-\pmb{c}_{j}\rVert^2)^{-1}}. 
\end{equation}
In general, high probability components in pseudo soft assignments represent high confidence.
To increase the discriminability of the pseudo soft assignments, we enhance them to obtain a unified target distribution (denoted as $P$) by:
\begin{equation}\label{eq:p}        
    p_{ij}=\frac{{(s_{ij})^2}/{\sum_i s_{ij}}}
    {\sum_{j}({(s_{i{j}})^2}/{\sum_i s_{i{j}}})}.
\end{equation}
To lead all autoencoders to learn higher discriminative embedded features, $P$ is used in all views' clustering-oriented loss function. Specifically, for the $v$-th view, the clustering loss $\mathcal{L}_c^v$ is the Kullback-Leibler divergence ($D_{KL}$) between the unified target distribution $P$ and its own cluster assignment distribution $Q^v$:
\begin{equation}\label{eq:Lc}
    \mathcal{L}_c^v= D_{KL}(P||Q^v)= \sum_{i=1}^N\sum_{j=1}^K p_{ij}\log{\frac{p_{ij}}{q_{ij}^v}}.
\end{equation}

So, the total loss of each view consists of two parts:
\begin{equation}\label{eq:LOSS}
\begin{aligned}
    \mathcal{L}^v= \mathcal{L}_r^v + \gamma \mathcal{L}_c^v.  
\end{aligned}
\end{equation}
\noindent where $\gamma$ is the trade-off coefficient. The reconstruction loss $\mathcal{L}_r^v$ (Eq.~(\ref{eq:Lr})) can be regarded as the regularization of embedded features, which ensures that the low-dimensional features can maintain the representation capability for examples. The optimization of clustering loss $\mathcal{L}_c^v$ makes the $v$-th view's autoencoder learn more discriminative features. 

Consequently, after optimizing ${\{\mathcal{L}^1, \mathcal{L}^2, \dots, \mathcal{L}^V\}}$, more discriminative information contained in multiple views can be mined. We further leverage the learned features to produce a more accurate target distribution with Eqs.~(\ref{eq:gf})--(\ref{eq:p}). 
Therefore, by performing the proposed feature learning, global discriminative information can be used to iteratively lead all views to learn more discriminative features, especially for those views with low discriminative features. Eventually, the discriminability of all views' embedded features is improved and thus clearer clustering structures can be obtained.

\subsection{The generalization bound}
To derive the generalization bound of the proposed method, we introduce the hypothesis function $h:\mathcal{X}\to \mathbb{R}^{D_v} \times \mathbb{R}^{d_v} \times \mathbb{R}^{d_v}$ that maps the given sample point into the reconstruction sample, embedding point and cluster centroid, where $\mathcal{X}$ denotes the input space. The family of $h$ is denoted as $\mathcal{H}$. Then the reconstruction sample, embedding point and cluster centroid are given by $\hat{\bm{x}}^v_i:=h_{\bm z}(\bm{x}^v_i) \in \mathbb{R}^{D_v}$, ${\bm{z}}^v_i:=h_{\bm z}(\bm{x}^v_i) \in \mathbb{R}^{d_v}$, and ${\bm \mu}:= h_{\bm \mu} \in \mathbb{R}^{d_v}$, respectively. Then the unified target distribution is formulated as $p_{ij}:=p(h_{\bm z}(\bm{x}_i),{h_{\bm \mu}}_j)$. The loss of the $v$-th view is
\begin{equation}\small
\begin{aligned}
    & \mathcal{L}^v_c = \frac{1}{N}\sum_{i=1}^N \Vert \bm{x}^v_i - h_{\bm z}(\bm{x}^v_i) \Vert^2 \\
    & + \frac{\gamma}{N}\sum_{i=1}^N \sum_{j=1}^K p(h_{\bm z}(\bm{x}_i),{h_{\bm \mu}}_j) \log \frac{p(h_{\bm z}(\bm{x}_i),{h_{\bm \mu}}_j)}{q^v_{ij}} \\
    & = \frac{1}{N}\sum_{i=1}^N \Vert \bm{x}^v_i - h_{\bm z}(\bm{x}^v_i) \Vert^2 \\
    & + \frac{\gamma}{N}\sum_{i=1}^N \sum_{j=1}^K p(h_{\bm z}(\bm{x}_i),{h_{\bm \mu}}_j) \log p(h_{\bm z}(\bm{x}_i),{h_{\bm \mu}}_j) \\
    & - \frac{\gamma}{N}\sum_{i=1}^N \sum_{j=1}^K p(h_{\bm z}(\bm{x}_i),{h_{\bm \mu}}_j) \log q^v_{ij}.
\end{aligned}
\end{equation}
According to the definition of $q$, the empirical risk is defined as
\begin{equation}\small
\begin{aligned}
    & \hat{\mathcal{L}}(h) = \frac{1}{N}\sum_{i=1}^N \Vert \bm{x}^v_i - h_{\bm z}(\bm{x}^v_i) \Vert^2 \\
    & + \frac{\gamma}{N}\sum_{i=1}^N \sum_{j=1}^K p(h_{\bm z}(\bm{x}_i),{h_{\bm \mu}}_j) \log p(h_{\bm z}(\bm{x}_i),{h_{\bm \mu}}_j)\\
    & + \frac{\gamma}{N}\sum_{i=1}^N \sum_{j=1}^K p(h_{\bm z}(\bm{x}_i),{h_{\bm \mu}}_j) \log \left(1+\Vert h_{\bm{z}}(\bm{x}^v_i)-{h_{\bm \mu}}^v_j\Vert^2\right) \\
    & + \frac{\gamma}{N}\sum_{i=1}^N \log \left( \sum_{j=1}^K {(1+\Vert h_{\bm z}(\bm{x}^v_i)-{h_{\bm \mu}}^v_j\Vert^2)}^{-1} \right).
\end{aligned}
\end{equation}
\begin{theorem}\label{the:t2}
Let $\mathcal{L}(h)$ be the expectation of $\hat{\mathcal{L}}(h)$. Suppose that for any $\bm x \in \mathcal{X}$ and $h\in \mathcal{H}$, there exists $M  < \infty$ such that $\Vert{\bm x}^v\Vert, \Vert h_{\bm x}({\bm x})\Vert, \Vert h_{\bm z}({\bm x}) \Vert, \Vert {h_{\bm \mu}} \Vert \in [0, M] $ hold. With probability $1-\delta$ for any $h \in \mathcal{H}$, the following inequality holds
\begin{equation}\small
    \hat{\mathcal{L}}(h) \leq \mathcal{L}(h) + \frac{c_1}{\sqrt{N}} + c_2\sqrt{\frac{\log \frac{1}{\delta}}{2N}},
\end{equation}
where $c_1$ and $c_2$ are constants that depend on $K$, $M$ and $\gamma$.
\end{theorem}
The proof of Theorem~\ref{the:t2} is shown in Appendix, which demonstrates that the expected clustering risk of the proposed method is bounded by sample-dependent complexity terms and the empirical risk of the proposed model on training samples with high probability. Since lower risk represents superior clustering performance, the proposed model is guaranteed to achieve promising clustering performance on new unseen multi-view samples.

\subsection{Consistent Multi-View Clustering}
In the $v$-th view, the clustering prediction of the $i$-th example is calculated by
\begin{equation}\label{eq:yv}
    y_i^v=\arg\mathop{\max_j}{(q_{ij}^v)}.
\end{equation}
On account of no label information in clustering, we do not even know which view's clustering prediction is more accurate. But, the following theorem guarantees the multi-view clustering consistency, i.e., multiple views have consistent clustering predictions for the same examples.
\begin{theorem}
Minimizing multiple KL divergence with a unified $P$ makes the $Q^v$ of multiple views tend to be consistent.
\label{the:t1}
\end{theorem}
\begin{proof}
The optimization of multi-view clustering loss of our proposed SDMVC is:
\begin{equation}
    \mathop{\min_{Q^v}}{\mathcal{L}_c^v = D_{KL}(P||Q^v), v \in {\{1, 2, \dots, V\}}}.
\end{equation}
For any two views, $a$ and $b$ $\in {\{1, 2, \dots, V\}}$, let $\xi_a$ and $\xi_b$ denote the optimization error of $\mathcal{L}_c^a$ and $\mathcal{L}_c^b$, respectively. Given non-negative of KL divergence, we obtain the following two inequalities:
\begin{equation}\label{eq:KLa}
    0 \leq D_{KL}(P||Q^a)= \sum_{i=1}^N\sum_{j=1}^K p_{ij}\log{\frac{p_{ij}}{q_{ij}^a}} \leq \xi_a,
\end{equation}
\begin{equation}\label{eq:KLb}
    0 \leq D_{KL}(P||Q^b)= \sum_{i=1}^N\sum_{j=1}^K p_{ij}\log{\frac{p_{ij}}{q_{ij}^b}} \leq \xi_b.
\end{equation}
By subtracting Eq.~(\ref{eq:KLa}) from Eq.~(\ref{eq:KLb}) (i.e., Eq.~(\ref{eq:KLb})$-$Eq.~(\ref{eq:KLa})), the following inequality holds:
\begin{equation}
    -\xi_a \leq \sum_{i=1}^N\sum_{j=1}^K p_{ij}\log{\frac{q_{ij}^a}{q_{ij}^b}} \leq \xi_b.
\end{equation}
When $P$ is fixed, minimizing $\xi_a \to 0$ and $\xi_b \to 0$ result in
\begin{equation}
    \sum_{i=1}^N\sum_{j=1}^K p_{ij}\log{\frac{p_{ij}}{q_{ij}^a}} \to 0. \quad i.e. \quad Q^a \to P,
\end{equation}
\begin{equation}
    \sum_{i=1}^N\sum_{j=1}^K p_{ij}\log{\frac{p_{ij}}{q_{ij}^b}} \to 0. \quad i.e. \quad Q^b \to P,
\end{equation}
and
\begin{equation}
    \sum_{i=1}^N\sum_{j=1}^K p_{ij}\log{\frac{q_{ij}^a}{q_{ij}^b}} \to 0. \quad i.e. \quad Q^b \to Q^a.
\end{equation}
Therefore, $Q^a$ and $Q^b$ tend to be consistent with each other. This conclusion can be easily generalized to the case of multiple views.
\end{proof}

Theorem \ref{the:t1} indicates that our method can obtain consistent soft cluster assignments in multiple views. Based on this, averaging multiple soft cluster assignments can avoid the interference of a few wrong predictions and achieve the definite clustering predictions of higher confidence. Therefore, the final clustering prediction is calculated by
\begin{equation}\label{eq:s}
    y_i=\arg\mathop{\max_j}{(\frac{1}{V}\sum_{v=1}^V q_{ij}^v)}.
\end{equation}

Considering the diversity of different views, it is unreasonable to expect all views to have consistent predictions for all examples. We define the $i$-th example is aligned when $y_i^1=y_i^2=\dots=y_i^V$. Accordingly, we count the rate of aligned examples in all examples, called ``Aligned Rate'', to determine the stop condition of the model. Since it is also an unsupervised process, we can stop training when a high Aligned Rate is achieved so as to ensure the multi-view clustering consistency.

The proposed pipeline is illustrated in Fig. \ref{fig:framework}. In conclusion, the multi-view clustering consistency is achieved by setting a unified target distribution. As optimizing $\mathcal{L}^v$ only affects the autoencoder and clustering layer of the $v$-th view, the optimization of each view's embedded features and cluster centroids is independent for other views. As a result, the proposed framework can achieve the clustering consistency across all views in $\{Q^v\}_{v=1}^V$ while preserving their features' diversity in $\{\pmb{z}^v\}_{v=1}^V$. This design is conducive to reduce the negative influence of the views with unclear clustering structures.
\subsection{Optimization}\label{sec:op}
In the beginning, the parameters of autoencoders are initialized randomly. To obtain effective target distribution, the autoencoders are pre-trained by Eq.~(\ref{eq:Lr}). After that, $K$-means is applied to initialize the learnable cluster centroids $\{\{\pmb{\mu}_j^v\}_{j=1}^K\}_{v=1}^V$. For the $v$-th view, the parameters to be trained are $\theta^v$ of encoder, $\phi^v$ of decoder, and $\pmb{\mu}_j^v$ of clustering layer. 

The target distribution $P$ is fixed during the discriminative feature learning. The gradients of $\mathcal{L}_c^v$ corresponding to cluster centroids $\pmb{\mu}_j^v$ and embedded features $\pmb{z}_i^v$ are
\begin{equation}
    \frac{\partial \mathcal{L}_c^v}{\partial \pmb{\mu}_j^v} = 2 \sum_{i=1}^N (1 + \left\| \pmb{z}_i^v - \pmb{\mu}_j^v \right\|^2 )^{-1} (q_{ij}^v-p_{ij}) (\pmb{z}_i^v - \pmb{\mu}_j^v)
\end{equation}
\noindent and
\begin{equation}
    \frac{\partial \mathcal{L}_c^v}{\partial \pmb{z}_i^v} = 2 \sum_{j=1}^K (1 + \left\| \pmb{z}_i^v - \pmb{\mu}_j^v \right\|^2 )^{-1} (p_{ij}-q_{ij}^v) (\pmb{z}_i^v - \pmb{\mu}_j^v).
\end{equation}

We use mini-batch gradient descent optimization and backpropagation algorithms to train the model. Let $n$ and $\lambda$ denote the batch size and the learning rate, respectively. The method is summarized in Algorithm \ref{alg:alg1}.
After a fixed iterations, the target distribution will be updated so that the autoencoders can learn more discriminative features. Please refer to Section \ref{sec:Details} for specific experimental settings.
\begin{algorithm}[!t]\caption{: Self-Supervised Discriminative Feature Learning for Deep Multi-View Clustering (SDMVC)} 
\label{alg:alg1}
\begin{algorithmic}[1]
\REQUIRE ~~\\
Multi-view dataset; Number of clusters $K$; \\
Trade-off coefficient $\gamma$; Stop threshold $\delta$.\\
\ENSURE ~~\\
Cluster assignments $\pmb{y}$.\\
\STATE Pre-train each view's deep autoencoder by Eq.~(\ref{eq:Lr}).
\STATE Initialize each view's cluster centroids by $K$-means.\\
\WHILE{not reaching the stop condition}
\STATE Calculate centroids on global features by Eqs.~(\ref{eq:gf}, \ref{eq:kmeans}).
\STATE Update the target distribution $P$ by Eqs.~(\ref{eq:pk}, \ref{eq:p}).
\STATE Count the Aligned Rate of $\{\pmb{y}^1, \pmb{y}^2, \dots, \pmb{y}^V\}$.
\IF{Aligned Rate $> \delta$}
\STATE Stop training.
\ENDIF
\FOR{fixed target distribution $P$}
\STATE Fine-tune all autoencoders' parameters by: 
\STATE \quad $\pmb{\mu}_j^v = \pmb{\mu}_j^v - \frac{\lambda}{n} \sum_{i=1}^n \frac{\partial \mathcal{L}_c^v}{\partial \pmb{\mu}_j^v}$,
\STATE \quad $\phi^v = \phi^v - \frac{\lambda}{n} \sum_{i=1}^n \frac{\partial \mathcal{L}_r^v}{\partial \phi^v}$, and
\STATE \quad $\theta^v = \theta^v - \frac{\lambda}{n} \sum_{i=1}^n \left(\frac{\partial \mathcal{L}_r^v}{\partial \theta^v} + \gamma \frac{\partial \mathcal{L}_c^v}{\partial \theta^v} \right)$.
\ENDFOR
\ENDWHILE
\STATE Output $\pmb{y}$ calculated by Eq.~(\ref{eq:s}).
\end{algorithmic}
\end{algorithm}

\textbf{Complexity Analysis}. $K, V,$ and $N$ are the number of clusters, views, and examples, respectively. Let $H$ denote the maximum number of neurons in autoencoders' hidden layers and $Z$ denote the maximum dimensionality of embedded features. Generally $V, K, Z \ll H, N$ holds. Both complexities of $K$-means and calculating the target distribution are $O(NZK)$. The complexity to count the Aligned Rate is $O(NK)$. The complexity of $V$ autoencoders is $O(VNH^2)$. In conclusion, the complexity of our algorithm is linear to the data size $N$.

\section{Experimental Setup}\label{sec:experiments}
\subsection{Datasets}\label{sec:Datasets}
The experiments are conducted on the following public datasets for comparison and analysis.

\begin{itemize}
    \item \textbf{MNIST-USPS}: MNIST\footnote{http://yann.lecun.com/exdb/mnist}~\cite{lecun1998gradient} and USPS\footnote{http://www.cad.zju.edu.cn/home/dengcai/Data/MLData.html}~\cite{hull1994database} are two image datasets about handwritten digits with two styles. These two datasets are always treated as two different views of digits in multi-view clustering. The same dataset as \cite{peng2019comic} is used in our experiment, where different views of a sample are different digits of the same category. Each view contains 10 categories and every category provides 500 examples.
    \item \textbf{Fashion-MV}: Fashion\footnote{https://github.com/zalandoresearch/fashion-mnist}~\cite{xiao2017fashion} contains 10 kinds of fashion products (such as T-shirt, Dress, and Coat). We follow the same manner to construct MNIST-USPS and use 30,000 examples to construct a multi-view version named Fashion-MV. It has three views and each of which consists of 10,000 gray images. Every three images sampled from the same category constitute three views of one instance.
    \item \textbf{BDGP}: Berkeley Drosophila Genome Project, i.e., BDGP\footnote{https://ranger.uta.edu/~heng/Drosophila}~\cite{cai2012joint} contains 2,500 images about drosophila embryos belonging to 5 categories, where 1,750-dimensional/dim visual feature and 79-dim textual feature of each image are extracted for multi-view clustering.
    \item \textbf{Caltech101-20}: \cite{zhang2018binary} constructs a multi-view dataset with 2,386 images sampled from a popular RGB image dataset Caltech101\footnote{http://www.vision.caltech.edu/Image\_Datasets/Caltech101}~\cite{fei2004learning}. It contains 20 categories and six different views, i.e., 48-dim Gabor, 40-dim wavelet moments (WM), 254-dim CENTRIST, 1,984-dim HOG, 512-dim GIST, and 928-dim LBP.
\end{itemize}

\begin{table*}[!t]
\centering
\renewcommand\tabcolsep{7.0pt} 
\begin{threeparttable}
    \begin{tabular}{r|ccc|ccc|ccc|ccc}
    \toprule
    &\multicolumn{3}{c|}{MNIST-USPS} &\multicolumn{3}{c|}{Fashion-MV} &\multicolumn{3}{c|}{BDGP} &\multicolumn{3}{c}{Caltech101-20}\cr
    \hline
    &\multicolumn{3}{c|}{2 views, $K=10$} &\multicolumn{3}{c|}{3 views, $K=10$} &\multicolumn{3}{c|}{2 views, $K=5$} &\multicolumn{3}{c}{6 views, $K=20$}\cr
    &\multicolumn{3}{c|}{5,000 examples} &\multicolumn{3}{c|}{10,000 examples} &\multicolumn{3}{c|}{2,500 examples} &\multicolumn{3}{c}{2,386 examples}\cr
    \hline
    Methods & ACC & NMI & ARI & ACC & NMI & ARI & ACC & NMI & ARI & ACC & NMI & ARI\\
    \hline
    $K$-means (1967) & 0.7678 & 0.7233 & 0.6353 & 0.7093 & 0.6561 & 0.5689 & 0.4324 & 0.5694 & 0.2604 & 0.4179 & 0.3351 & 0.2605 \\
    SC (2002) & 0.6596 & 0.5811 & 0.4864 & 0.5354 & 0.5772 & 0.4261 & 0.5172 & 0.5891 & 0.3156 & 0.4620 & 0.4589 & 0.3933 \\
    DEC (2016)  & 0.7310  & 0.7146 & 0.6323 & 0.6707 & 0.7234 & 0.6291 & 0.9478  & 0.8662 & 0.8702 & 0.4268  & 0.6251 & 0.3767 \\
    IDEC (2017) & 0.7658  & 0.7689 & 0.6801 & 0.6919 & 0.7501 & 0.6522 & 0.9596  & 0.8940 & 0.9025 & 0.4318  & 0.6253 & 0.3773 \\
    \hline
    BMVC (2018)    & 0.8802  & 0.8945 & 0.8448 & 0.7858 & 0.7488 & 0.6835 & 0.3492  & 0.1202 & 0.0833 & 0.5553  & 0.6203 & 0.5038 \\
    MVC-LFA (2019) & 0.7678  & 0.6749 & 0.6092 & 0.7910 & 0.7586 & 0.6887 & 0.5468  & 0.3345 & 0.2881 & 0.4221  & 0.5846 & 0.2994 \\
    GMC (2019)    & 0.9968 & 0.9903 & 0.9929 & 0.8321 & 0.8940 & 0.7749 & 0.5912 & 0.6261 & 0.4313 & 0.5671 & 0.5359 & 0.2284\\
    DAMC (2019)    & 0.7172 & 0.8085 & 0.6980 & 0.6492 & 0.6012 & 0.5487 & \underline{0.9822} & \underline{0.9461} & \underline{0.9437} & 0.3157 & 0.2849 & 0.2443\\
    SAMVC (2020)   & 0.6965  & 0.7458 & 0.6090 & 0.6286 & 0.6878 & 0.5665 & 0.5386  & 0.4625 & 0.2099 & 0.5218  & 0.5961 & 0.4653 \\
    CDIMC-net (2020) & 0.6203  & 0.6763 & 0.6338 & 0.7763 & 0.8104 & 0.7861 & 0.8827  & 0.7893 & 0.8194 & \underline{0.5815} & \underline{0.6903} & \underline{0.5652} \\
    EAMC (2020)    & 0.7304 & 0.8353 & 0.7215 & 0.6166 & 0.6113 & 0.5100 & 0.6756 & 0.4702 & 0.3931 & 0.3026  & 0.2633 & 0.2255\\
    SiMVC (2021)    & 0.9774 & 0.9630 & 0.9528 & 0.8272 & 0.8312 & 0.7761 & 0.6972 & 0.5326 & 0.4455 & 0.4208 & 0.6108 & 0.3343\\
    CoMVC (2021)    & 0.9847 & 0.9735 & 0.9801 & \underline{0.8498} & 0.8592 & \underline{0.7986} & 0.8068 & 0.6739 & 0.5928 & 0.4376 & 0.6149 & 0.3357\\
    DEMVC (2021)   & \underline{0.9976}  & \underline{0.9939} & \underline{0.9948} & 0.7864 & \underline{0.9061} & 0.7793 & 0.9548  & 0.8720 & 0.8901 & 0.5754  & 0.6874 & 0.5129 \\
    SDMVC (ours)   &\textbf{0.9982}&\textbf{0.9947} &\textbf{0.9960} &\textbf{0.8626}&\textbf{0.9215} &\textbf{0.8405} &\textbf{0.9835}&\textbf{0.9466} &\textbf{0.9551} &\textbf{0.7158}&\textbf{0.7176} &\textbf{0.7265} \\
    \bottomrule
    \end{tabular}
    \caption{Comparison results. The best values are highlighted in boldface and the second best values are highlighted in underline.
    }
    \label{tab:table1}
\end{threeparttable}
\end{table*}

\subsection{Comparing Methods}
\label{sec:Comparison}
We compare our SDMVC against the following traditional and state-of-the-art methods.

\begin{itemize}

\item \textbf{$K$-means}: the most classic clustering method~\cite{macqueen1967some}.

\item \textbf{SC}: spectral clustering~\cite{ng2002spectral}.

\item \textbf{DEC}: deep embedded clustering~\cite{xie2016unsupervised}.

\item \textbf{IDEC}: improved deep embedded clustering~\cite{guo2017improved}.

\item \textbf{BMVC}: binary multi-view clustering~\cite{zhang2018binary}.

\item \textbf{MVC-LFA}: multi-view clustering via late fusion alignment maximization~\cite{wang2019multi}.

\item \textbf{GMC}: graph-based multi-view clustering~\cite{wang2019gmc}.

\item \textbf{DAMC}: deep adversarial multi-view clustering network~\cite{li2019deep}.

\item \textbf{SAMVC}: self-paced and auto-weighted multi-view clustering~\cite{ren2020self}.

\item \textbf{CDIMC-net}: cognitive deep incomplete multi-view clustering network~\cite{wen2020cdimc}.

\item \textbf{EAMC}: end-to-end adversarial-attention network for multi–modal clustering~\cite{Zhou2020end}.

\item \textbf{SiMVC} and \textbf{CoMVC}: reconsidering representation alignment for multi-view clustering~\cite{trostenMVC}.

\item \textbf{DEMVC}: deep embedded multi-view clustering with collaborative training~\cite{xu2021deep}.
\end{itemize}

$K$-means, SC, DEC,and IDEC are single-view methods. For these single-view methods, the input is the concatenation of all views.
BMVC, MVC-LFA, GMC, DAMC, SAMVC, CDIMC-net, EAMC, SiMVC, CoMVC, and DEMVC belong to multi-view methods.
The shallow models contain $K$-means, SC, BMVC, MVC-LFA, GMC, and SAMVC. The deep models include DEC, IDEC, CDIMC-net, DAMC, EAMC, SiMVC, CoMVC, DEMVC, and our SDMVC.

\subsection{Implementation Details}
\label{sec:Details}
The fully connected (Fc) and convolutional (Conv) neural networks with general settings are both applied to test SDMVC. For BDGP and Caltech101-20, since all views of them are vector data, we use the same fully connected autoencoder (FAE) as \cite{guo2017improved,guo2018deep}. For each view, the encoder is: $\text{Input}-\text{Fc}_{500}-\text{Fc}_{500}-\text{Fc}_{2000}-\text{Fc}_{10}$. For MNIST-USPS and Fashion-MV, we follow \cite{ren2020deep,xu2021deep} and use the same type of convolutional autoencoder (CAE) for each view to learn embedded features. The encoder is: $\text{Input}-\text{Conv}_{32}^{5}-\text{Conv}_{64}^{5}-\text{Conv}_{128}^{3}-\text{Fc}_{10}$. It represents that convolution kernel sizes are 5-5-3 and channels are 32-64-128. The stride is 2. Decoders are symmetric with the encoders of corresponding views. The following settings are the same for all experimental datasets. The dimensionality of all views' embedded features are reduced to 10. ReLU \cite{glorot2011deep} is the activation function and Adam \cite{kingma2014adam} (default learning rate is 0.001) is chosen as the optimizer. The multiple views' autoencoders are pre-trained for 500 epochs. The trade-off coefficient $\gamma$ is set to 0.1. The batch size is 256. Update the target distribution after fine-tuning every 1000 batches. The stop condition is that the Aligned Rate reaches about $90\%$. The input features of all datasets are scaled to $[0, 1]$.

The experiments of SDMVC are conducted on Windows PC with GeForce RTX 2060 GPU (6GB caches) and Intel (R) Core (TM) i5-9400F CPU @ 2.90GHz, 16.0GB RAM.

The open-source codes and corresponding suggested settings of comparing methods are adopted. The hyper-parameters of them are as follows. Specifically, the trade-off coefficient $\gamma$ of IDEC and DEMVC is 0.1. 
For BMVC, $r$ is 5, $\beta$ is 0.003, $\gamma$ is 0.01, $\lambda$ is $10^{-5}$, and the length of code is 128. 
For MVC-LFA, Gaussian kernel is used and $\lambda$ is $2^3$. 
For GMC, $k$ and $\lambda$ are empirically set to 15 and 1.
The SPL controlling parameter $\lambda^{(v)}$ of SAMVC is set to add $15\%$ examples from each view in each iteration.
For CDIMC-net, the graph embedding parameter $\alpha$ is set to 0.0001.
Since the released codes of DAMC and CDIMC-net can not handle the dataset with six views, we extended them to test their performance.
The implementation settings of EAMC, SiMVC, and CoMVC come from https://github.com/DanielTrosten/mvc.

\subsection{Evaluation Measures}\label{sec:metrics}
The used quantitative metrics contain unsupervised clustering accuracy (ACC), normalized mutual information (NMI), and adjusted rand index (ARI).
The reported results are the average values of 10 runs.
Larger values of ACC/NMI/ARI indicate better clustering performance.

\section{Results and Analysis}\label{sec:Algorithm Results}
\begin{table}[!t]
\renewcommand\tabcolsep{3.5pt} 
    \begin{center}
        \newcommand{\minitab}[2][l]{\begin{tabular}{#1}#2\end{tabular}}
        \begin{tabular}{|c|c|c|c|c|}
            \hline
            Datasets & Methods & ACC & NMI & ARI\cr
            \hline
            \hline
            \multirow{5}{*}{\minitab[c]{MNIST-USPS \\ 2 views, $K=10$ \\ 5,000 examples}}
            & IDEC (view 1) & \textbf{0.8246}  & \textbf{0.7963} & \textbf{0.7292} \\
            & IDEC (view 2) & 0.5616 & 0.6308 & 0.4534 \\
            & SDMVC (view 1) & \textbf{0.9888}  & \textbf{0.9696} & \textbf{0.9752} \\
            & SDMVC (view 2) & 0.9978  & 0.9933 & 0.9951 \\
            & SDMVC & 0.9982   & 0.9947  & 0.9960 \\
            
            \hline
            \hline
            \multirow{7}{*}{\minitab[c]{Fashion-MV \\ 3 views, $K=10$ \\ 10,000 examples}}
            & IDEC (view 1) & 0.4918  & 0.5780 & 0.3977 \\
            & IDEC (view 2) & 0.4905 & 0.5961 & 0.4010 \\
            & IDEC (view 3) & \textbf{0.5117} & \textbf{0.6042} & \textbf{0.4103} \\
            & SDMVC (view 1) & 0.8438  & 0.8831 & 0.8049 \\
            & SDMVC (view 2) & 0.8465  & 0.8935 & 0.8151 \\
            & SDMVC (view 3) & \textbf{0.8477}  & \textbf{0.8865} & \textbf{0.8090} \\
            & SDMVC & 0.8626   & 0.9215  & 0.8405 \\
            
            \hline
            \hline
            \multirow{5}{*}{\minitab[c]{BDGP \\ 2 views, $K=5$ \\ 2,500 examples}}
            & IDEC (view 1) & 0.4628  & 0.2996 & 0.2492 \\
            & IDEC (view 2) & \textbf{0.9564} & \textbf{0.8867} & \textbf{0.8939} \\
            & SDMVC (view 1) & 0.9852  & 0.9535 & 0.9637 \\
            & SDMVC (view 2) & \textbf{0.9752}  & \textbf{0.9327} & \textbf{0.9393} \\
            & SDMVC & 0.9835   & 0.9466  & 0.9551 \\
            
            \hline
            \hline
            \multirow{12}{*}{\minitab[c]{Caltech101-20 \\ 6 views, $K=20$ \\ 2,386 examples}}
            & IDEC (view 1) & 0.2825  & 0.3338 & 0.1370 \\
            & IDEC (view 2) & 0.3282 & 0.4111 & 0.2364 \\
            & IDEC (view 3) & 0.3127  & 0.3592 & 0.1678 \\
            & IDEC (view 4) & \textbf{0.4531} & \textbf{0.6352} & \textbf{0.3872} \\
            & IDEC (view 5) & 0.3919  & 0.5830 & 0.3350 \\
            & IDEC (view 6) & 0.3822 & 0.5217 & 0.3159 \\
            & SDMVC (view 1) & 0.7154  & 0.7171 & 0.7269 \\
            & SDMVC (view 2) & 0.7137  & 0.7110 & 0.7236 \\
            & SDMVC (view 3) & 0.7158  & 0.7174 & 0.7273 \\
            & SDMVC (view 4) & \textbf{0.7167}  & \textbf{0.7216} & \textbf{0.7271} \\
            & SDMVC (view 5) & 0.7154  & 0.7180 & 0.7282 \\
            & SDMVC (view 6) & 0.7146  & 0.7183 & 0.7284 \\
            & SDMVC & 0.7158   & 0.7176  & 0.7265 \\

            \hline
        \end{tabular}
    \end{center}
    \caption{Improvements compared to individual views. We compare SDMVC with a single-view baseline (IDEC) in each view.}
    \label{tab:table2}
\end{table}

\subsection{Results on Real Data}\label{sec:results}
Our proposed framework is applicable to both convolutional and fully connected autoencoders. Therefore, image data is fed into convolutional autoencoders and vector data is fed into fully connected autoencoders. However, as some algorithms can not handle raw image data, the data is reshaped into vectors as input. The quantitative comparison is shown in Table \ref{tab:table1}. One can observe that: (1) Our SDMVC achieves the best performance on the quantitative metrics across all datasets. Meanwhile, we find the improvements are significant when $t$-test with $5\%$ significance level is used to evaluate the statistical significance. 
Especially on Caltech101-20, SDMVC improves existing methods by a large margin.
Although some methods have achieved comparable results on some datasets, they are not robust in terms of multiple datasets.
For example, DAMC obtains excellent clustering performance on BDGP but it does not perform well on other datasets.
SiMVC and CoMVC achieve good performance on MNIST-USPS and Fashion-MV, while showing poor performance on BDGP and Caltech101-20.
(2) In general, the clustering performance of single-view methods (i.e., $K$-means, SC, DEC, and IDEC) is worse than that of multi-view methods. However, the performance of many comparison MVC methods is also limited on BDGP and Caltech101-20.
Especially on BDGP, the performance of some MVC methods (e.g., EAMC and SiMVC) is lower than that of single-view methods (i.e., DEC and IDEC).
The reason is that the discriminability of multiple views have large gaps and certain views' clustering structures are highly unclear.
For example, Table \ref{tab:table2} shows the clustering performance of individual views evaluated by single-view method IDEC, where we could observe that the ACC values of the worst views among BDGP and Caltech101-20 are 0.4628 and 0.2825, respectively.
These views with unclear clustering structures cause negative effects in some multi-view clustering methods, leading to the unsatisfying clustering performances.
Yet even that, SDMVC has achieved the state-of-the-art clustering performance. The reason is that, firstly, each view's training is independent and the target distribution is generated from a global perspective. In our method, the negative effects caused by the low discriminative views could be overcomed, which is accordance with the analysis in our Eqs. (\ref{7}), (\ref{8}), and (\ref{9}).
Secondly, the consistent clustering for multiple views is achieved by the cluster assignments instead of the learned features.
In this way, the interference from multiple views can be alleviated and the diversity of their features is preserved, and thus more complementary information can be mined to boost clustering performance.

\begin{figure*}[!t]
\centering
\subfigure[View 1-0 (NMI = 0.296)]{\includegraphics[height=1.40in, width=1.6in]{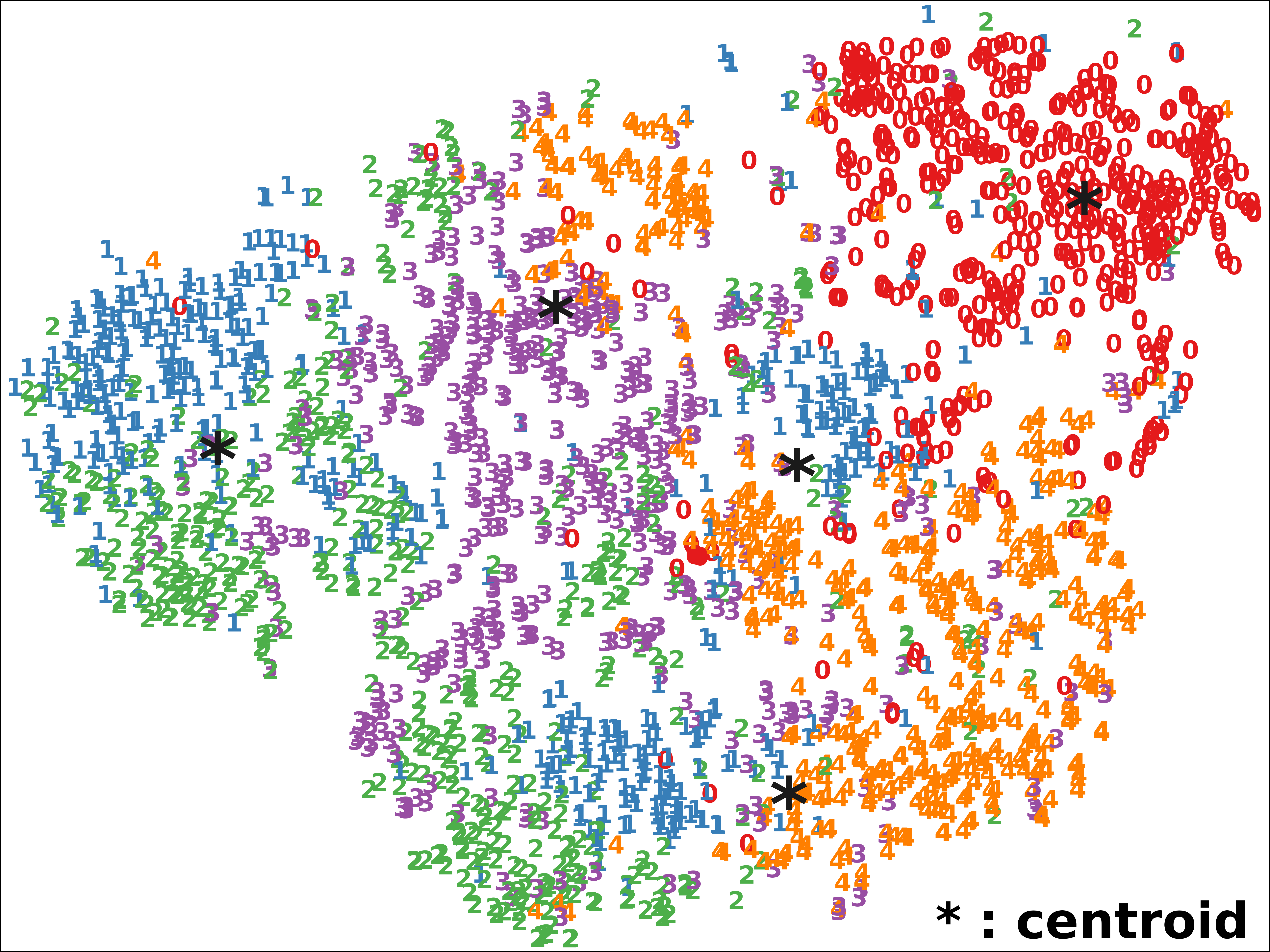}}
\subfigure[View 1-1000 (NMI = 0.781)]{\includegraphics[height=1.40in, width=1.6in]{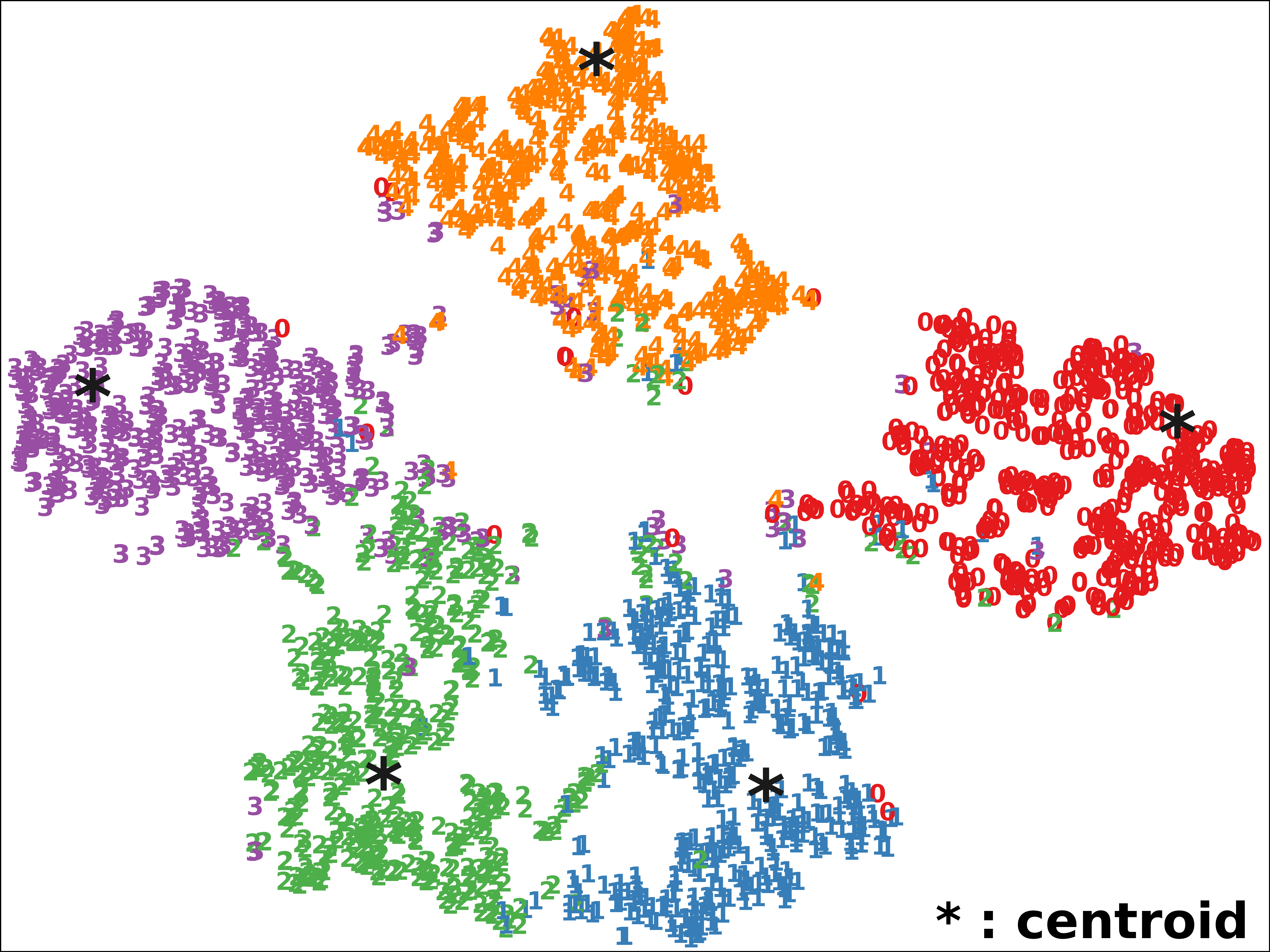}}
\subfigure[View 1-2000 (NMI = 0.914)]{\includegraphics[height=1.40in, width=1.6in]{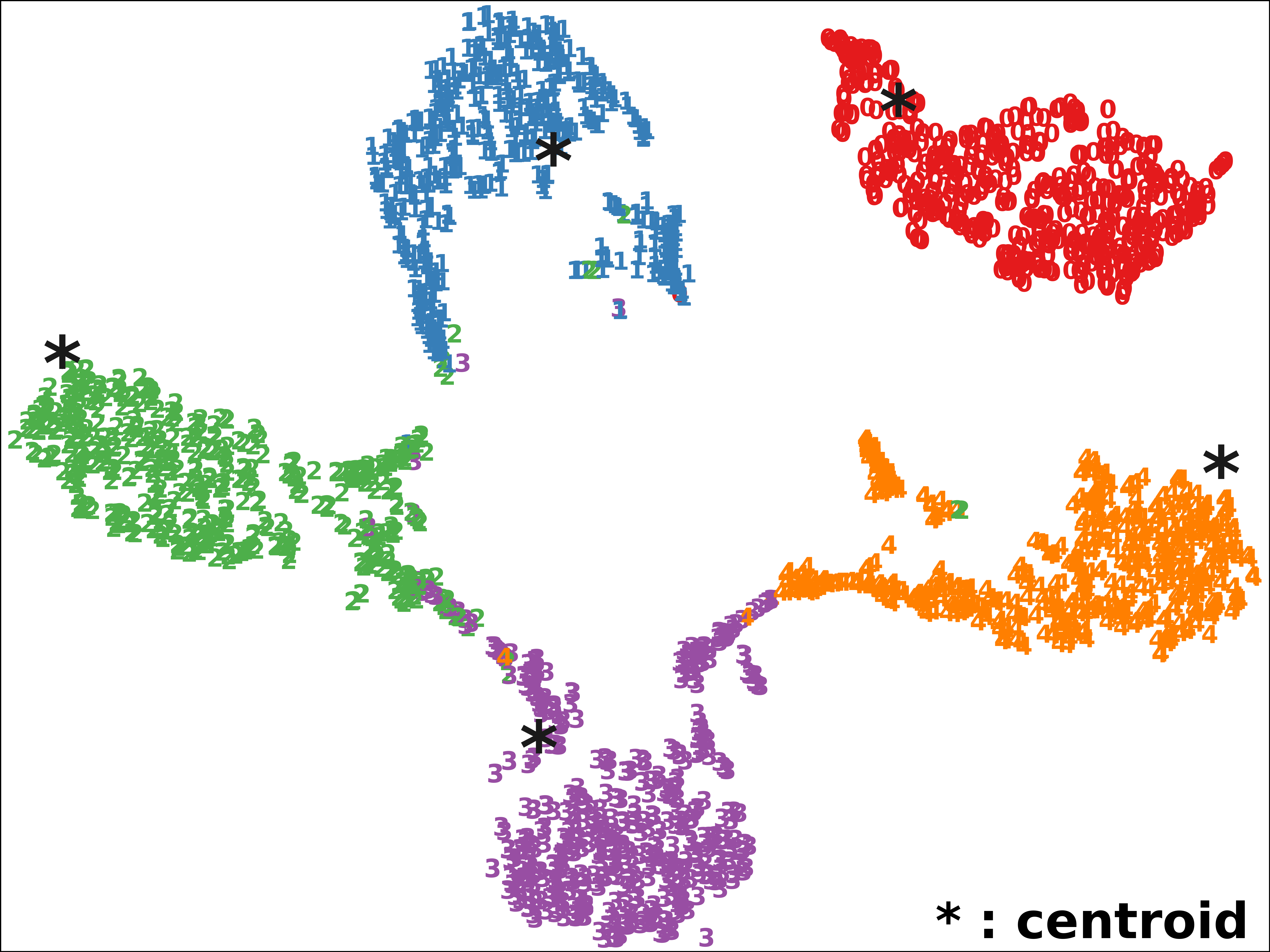}}
\subfigure[View 1-3000 (NMI = 0.946)]{\includegraphics[height=1.40in, width=1.6in]{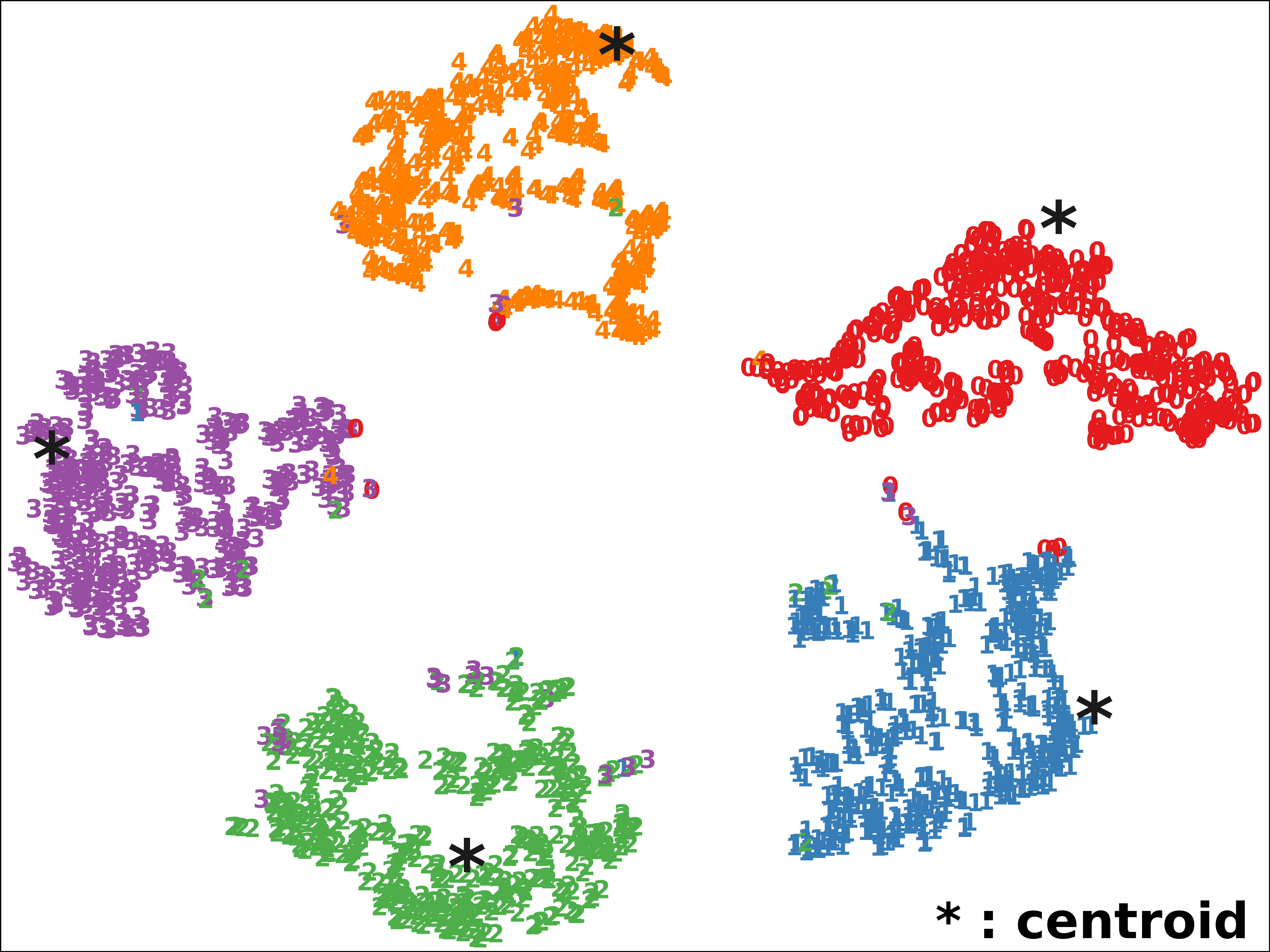}}
\vfill
\subfigure[View 2-0 (NMI = 0.699)]{\includegraphics[height=1.40in, width=1.6in]{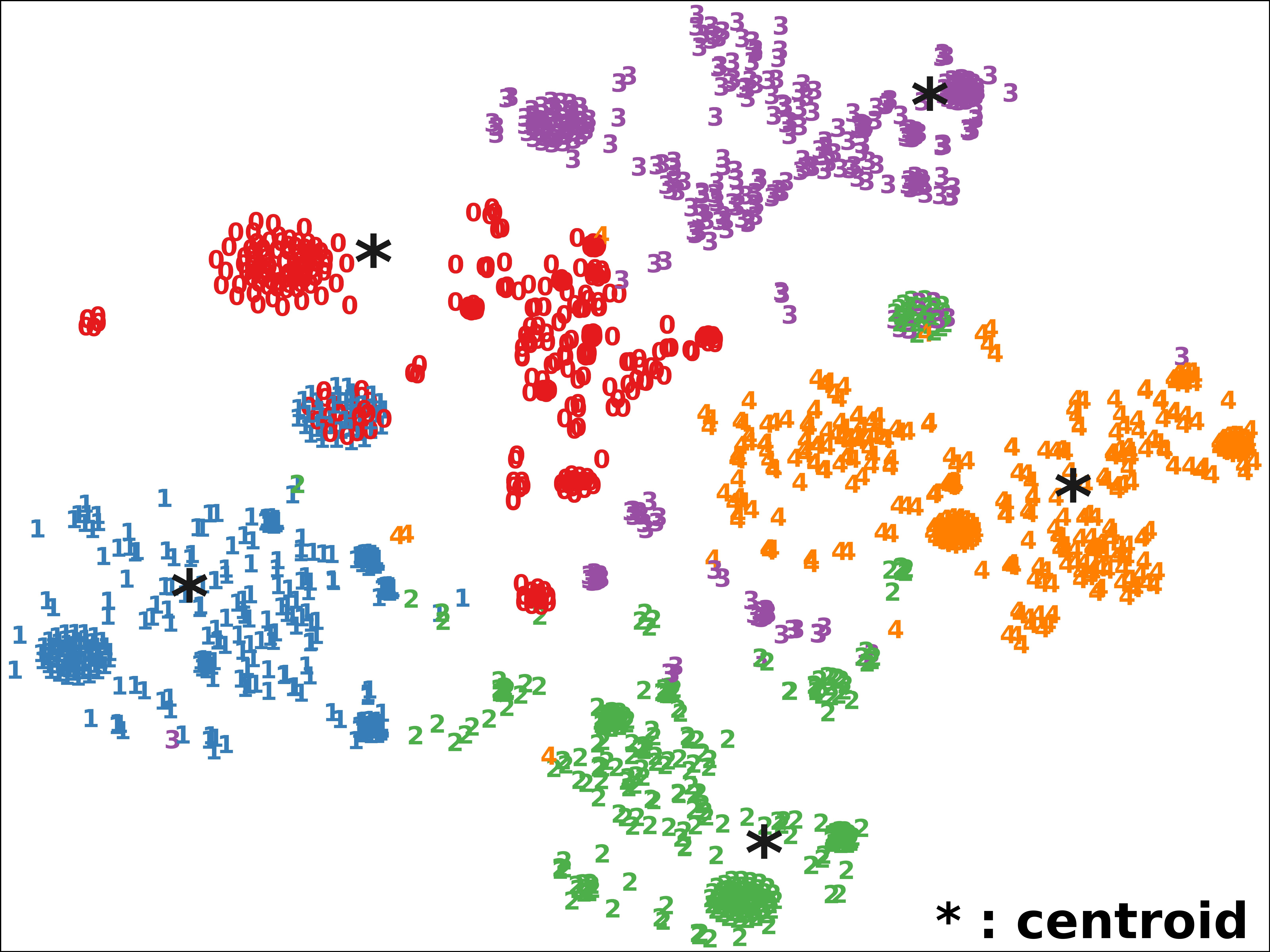}}
\subfigure[View 2-1000 (NMI = 0.892)]{\includegraphics[height=1.40in, width=1.6in]{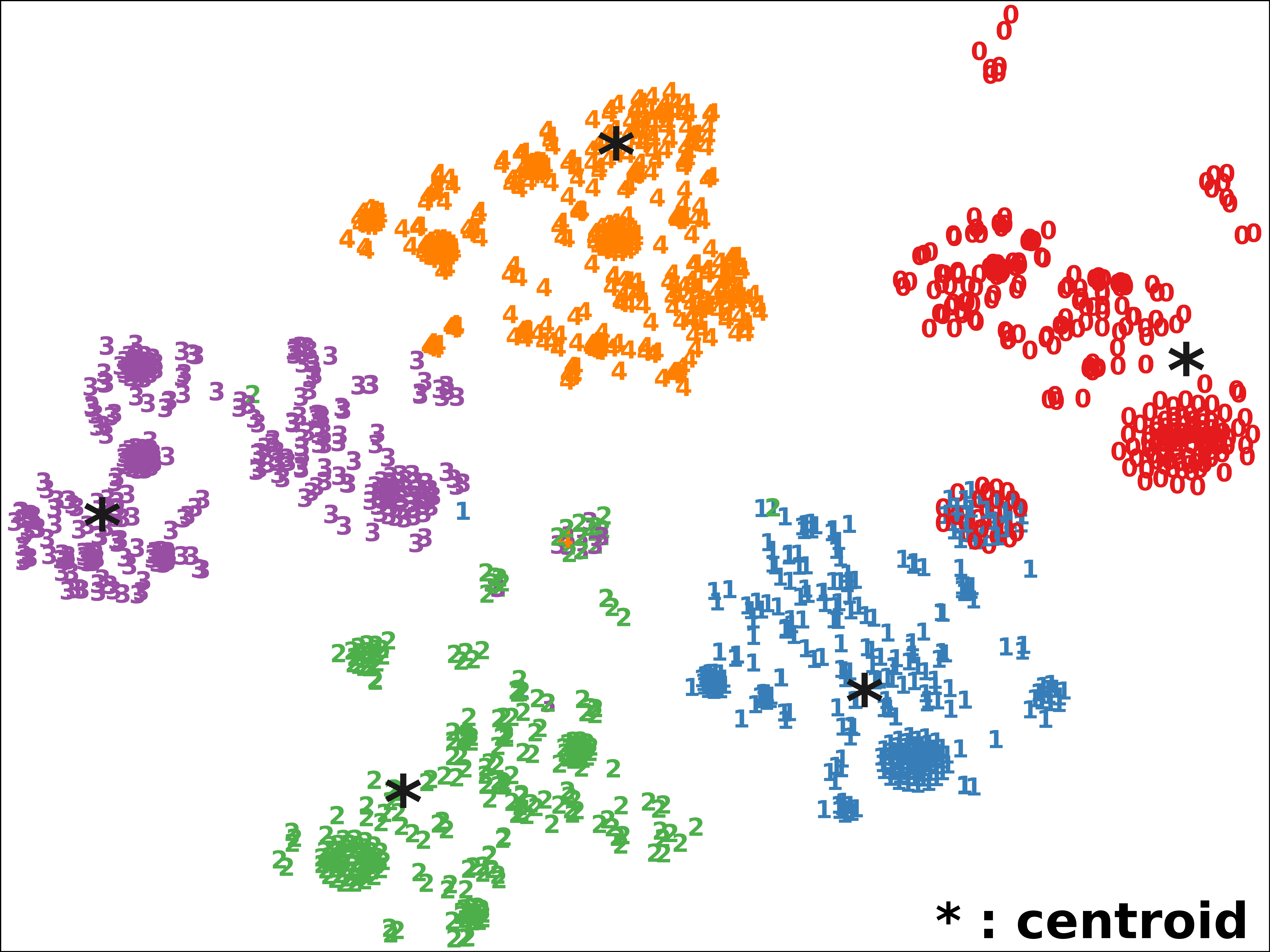}}
\subfigure[View 2-2000 (NMI = 0.923)]{\includegraphics[height=1.40in, width=1.6in]{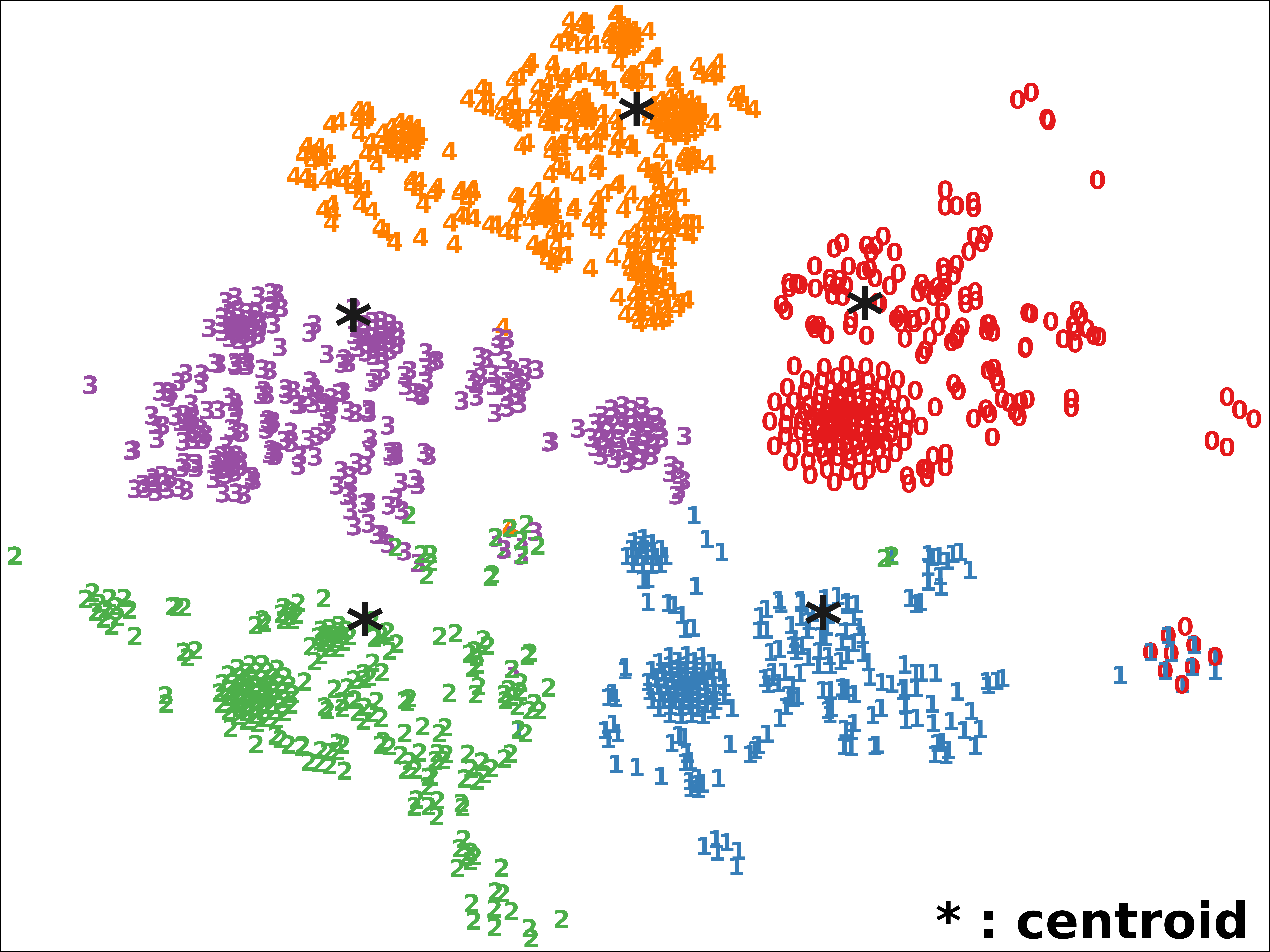}}
\subfigure[View 2-3000 (NMI = 0.928)]{\includegraphics[height=1.40in, width=1.6in]{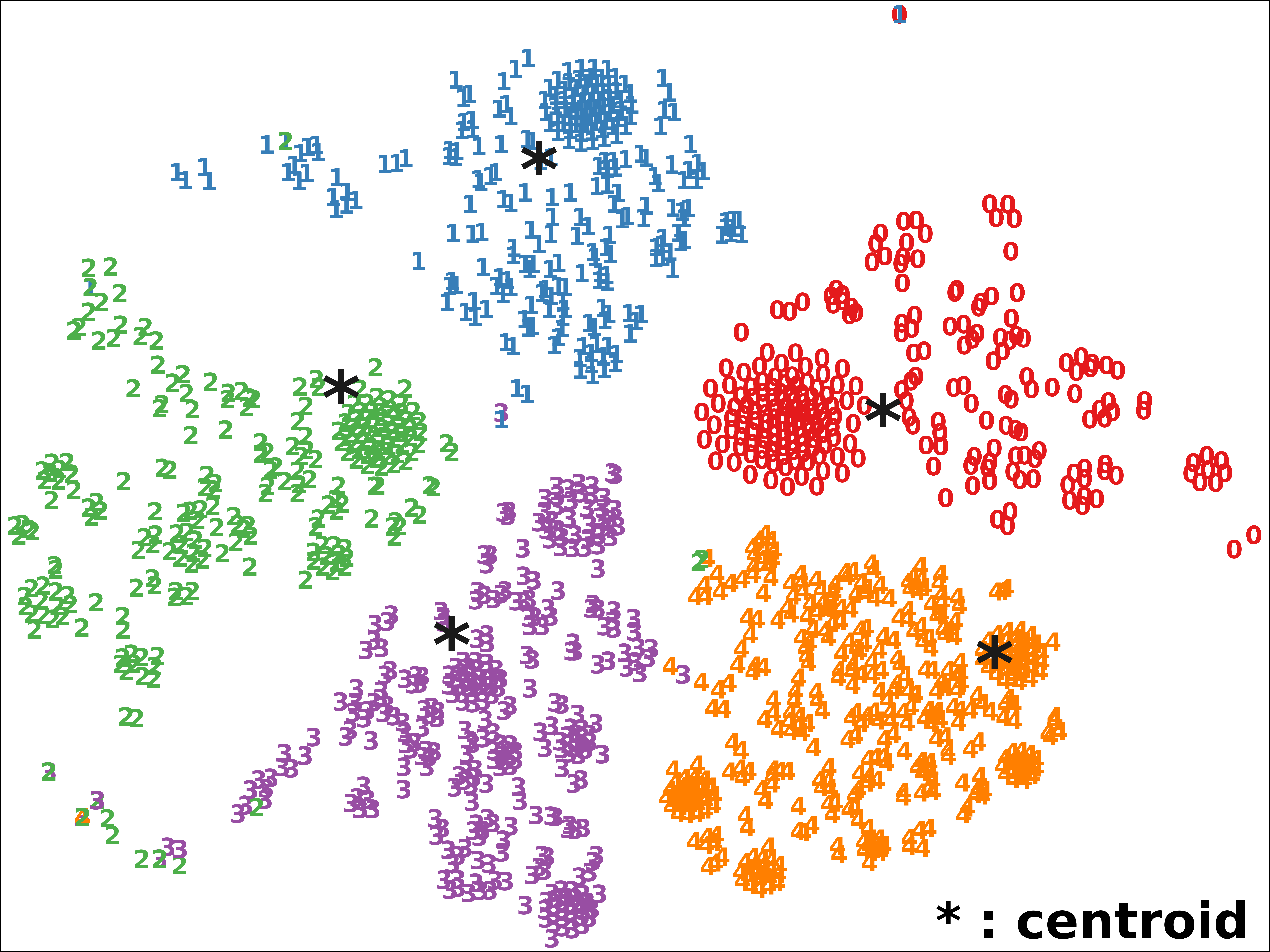}}
\vfill
\subfigure[Global-0 (NMI = 0.774)]{\includegraphics[height=1.40in, width=1.6in]{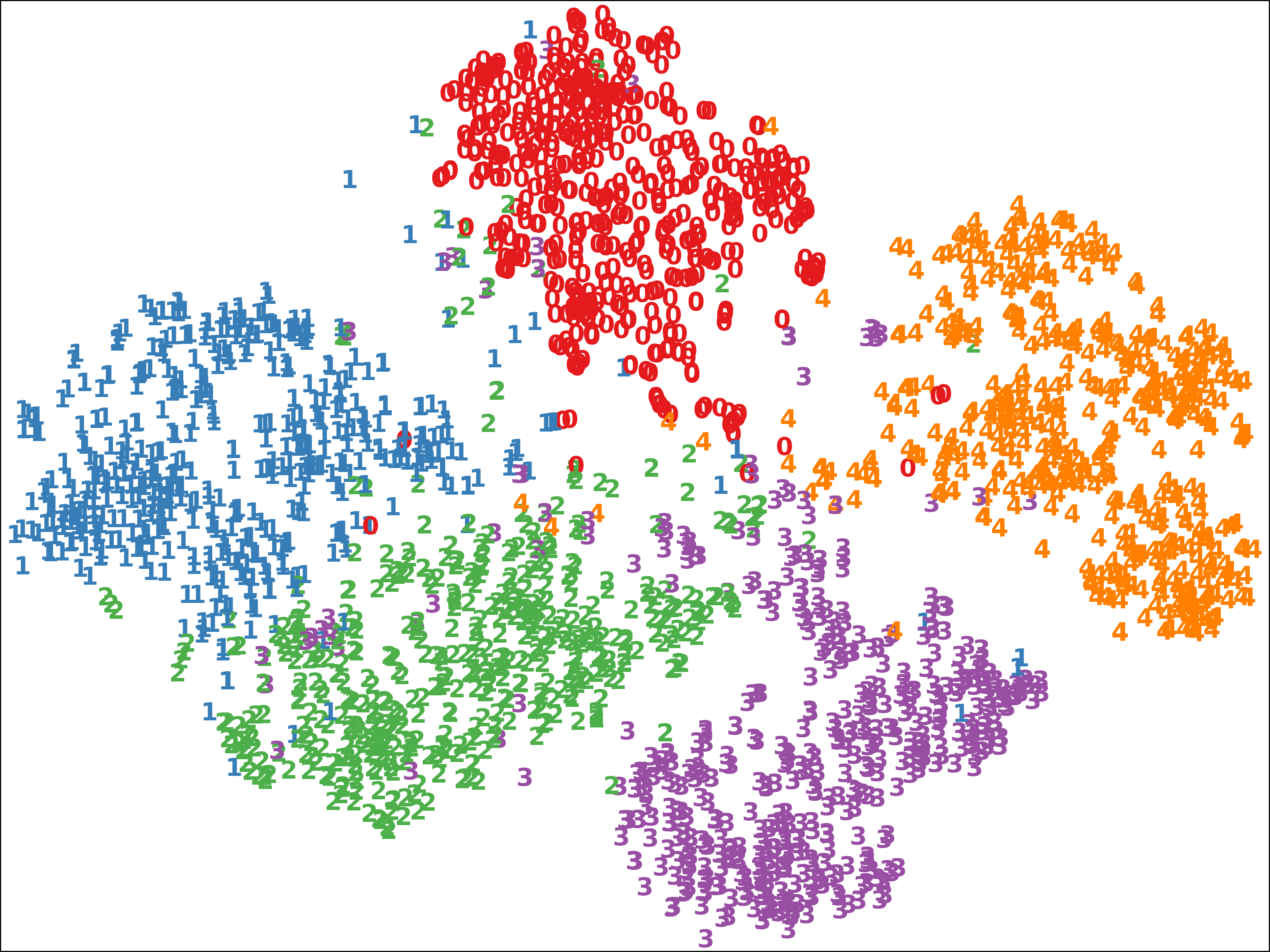}}
\subfigure[Global-1000 (NMI = 0.919)]{\includegraphics[height=1.40in, width=1.6in]{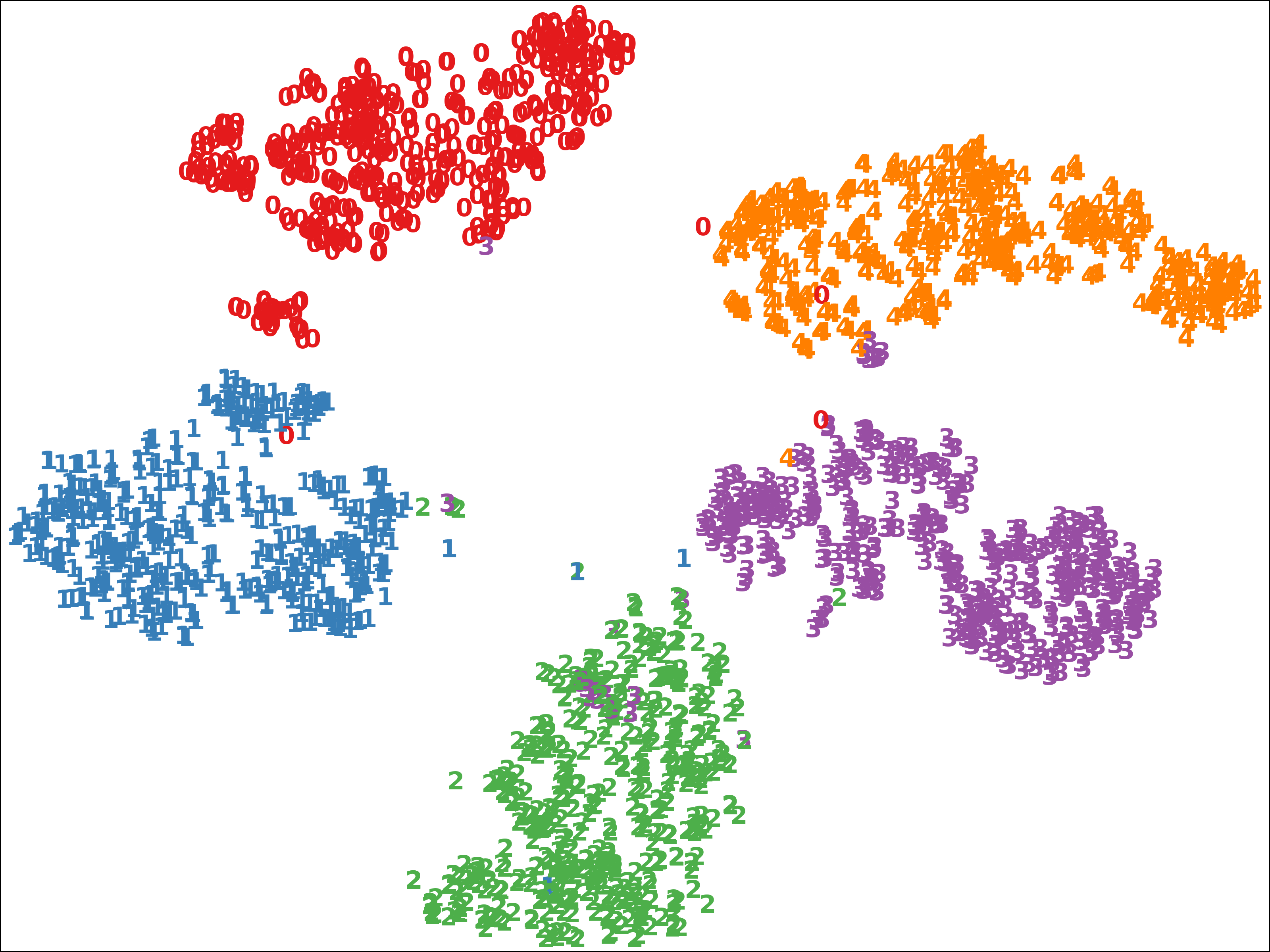}}
\subfigure[Global-2000 (NMI = 0.940)]{\includegraphics[height=1.40in, width=1.6in]{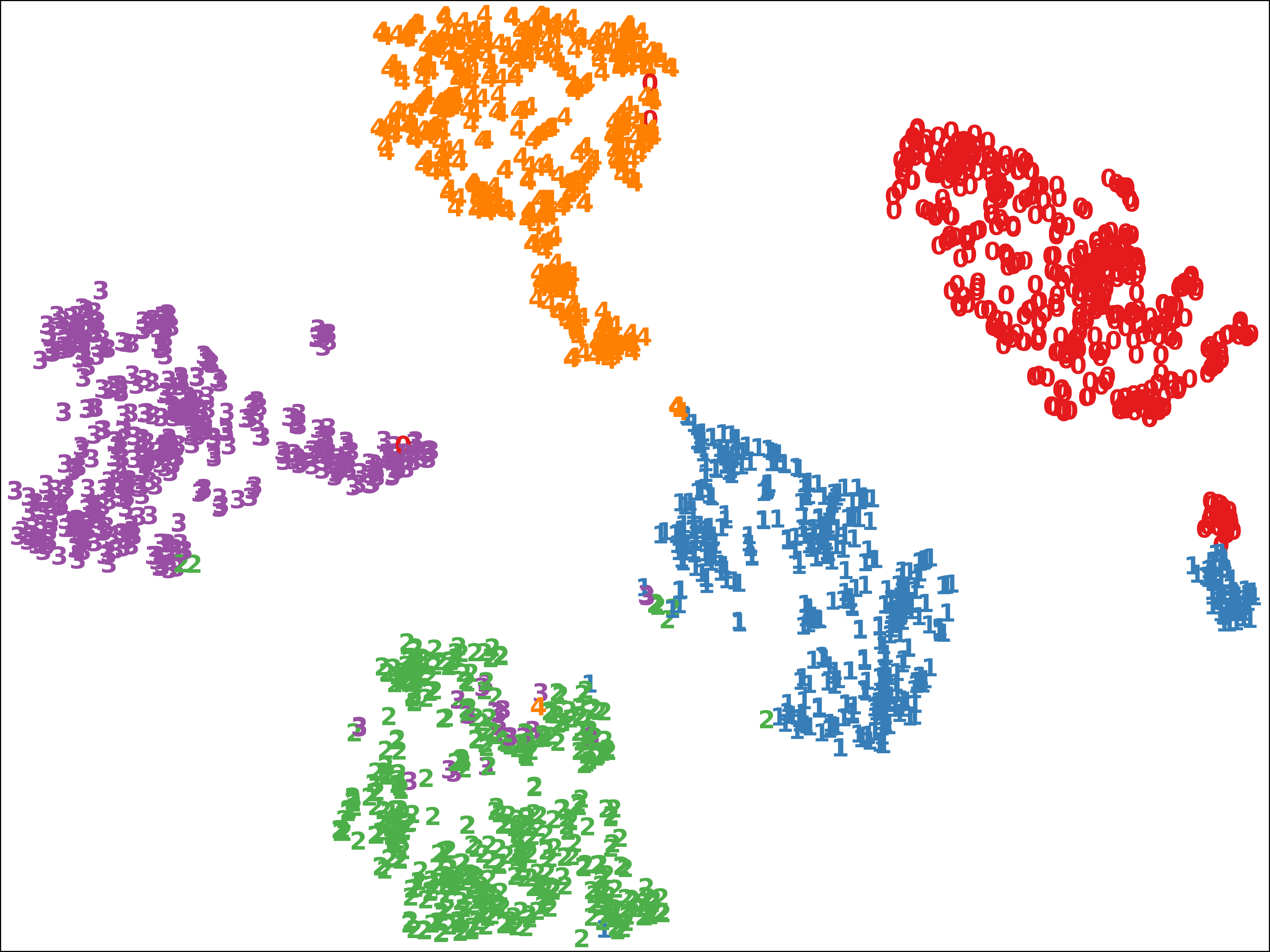}}
\subfigure[Global-3000 (NMI = 0.947)]{\includegraphics[height=1.40in, width=1.6in]{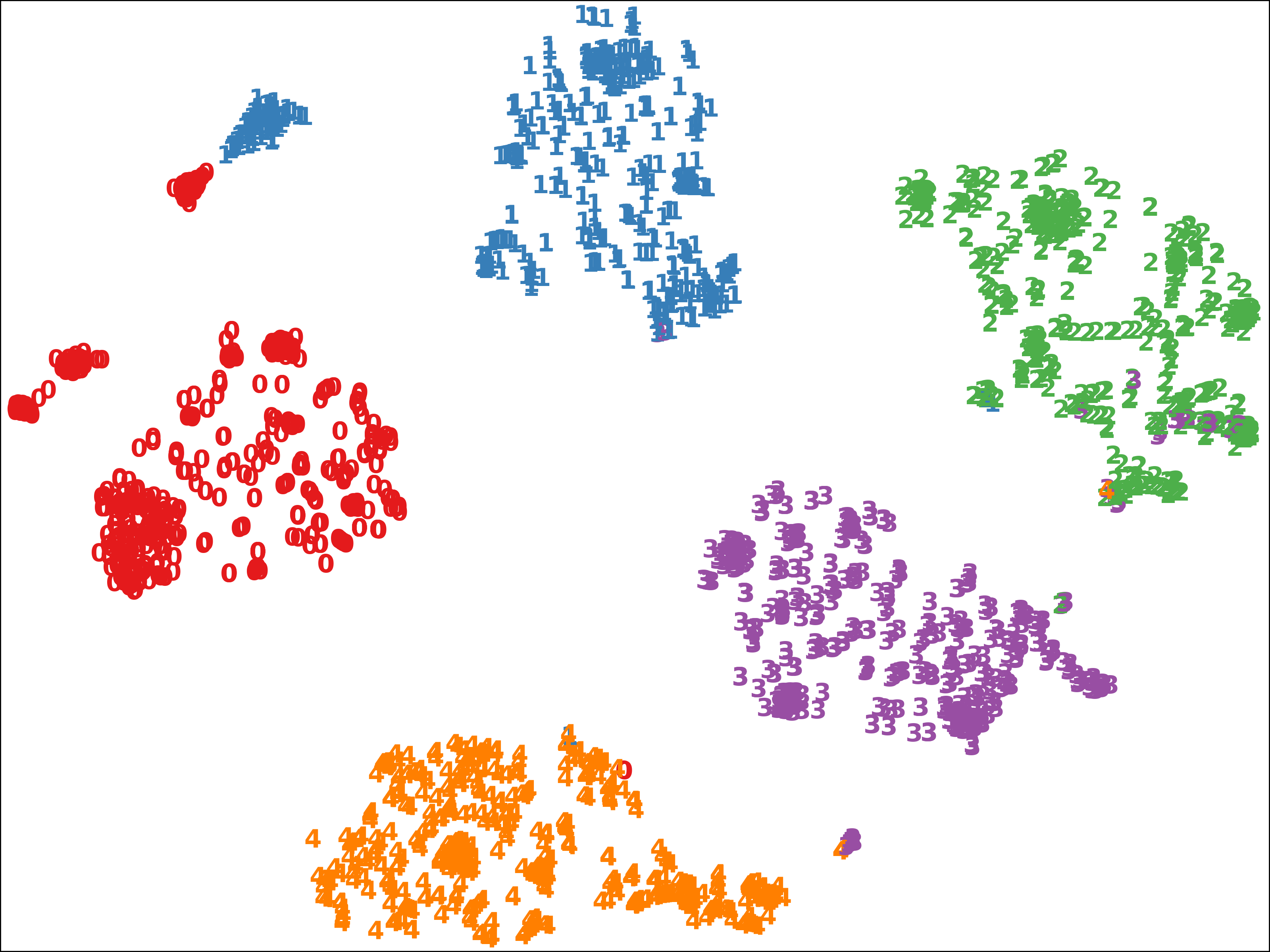}}
\caption{Visualization of embedded features and global features on BDGP. View 1 is visual features (a-d), View 2 is textual features (e-h), and Global denotes global features (i-l).
(a) and (e) show the embedded features when pre-training is finished.
The numbers represent batches that have been trained in the subsequent discriminative feature learning.}
\label{fig:separation}
\end{figure*}

\begin{figure}[!t]
\centering
\subfigure{\includegraphics[height=2.1in]{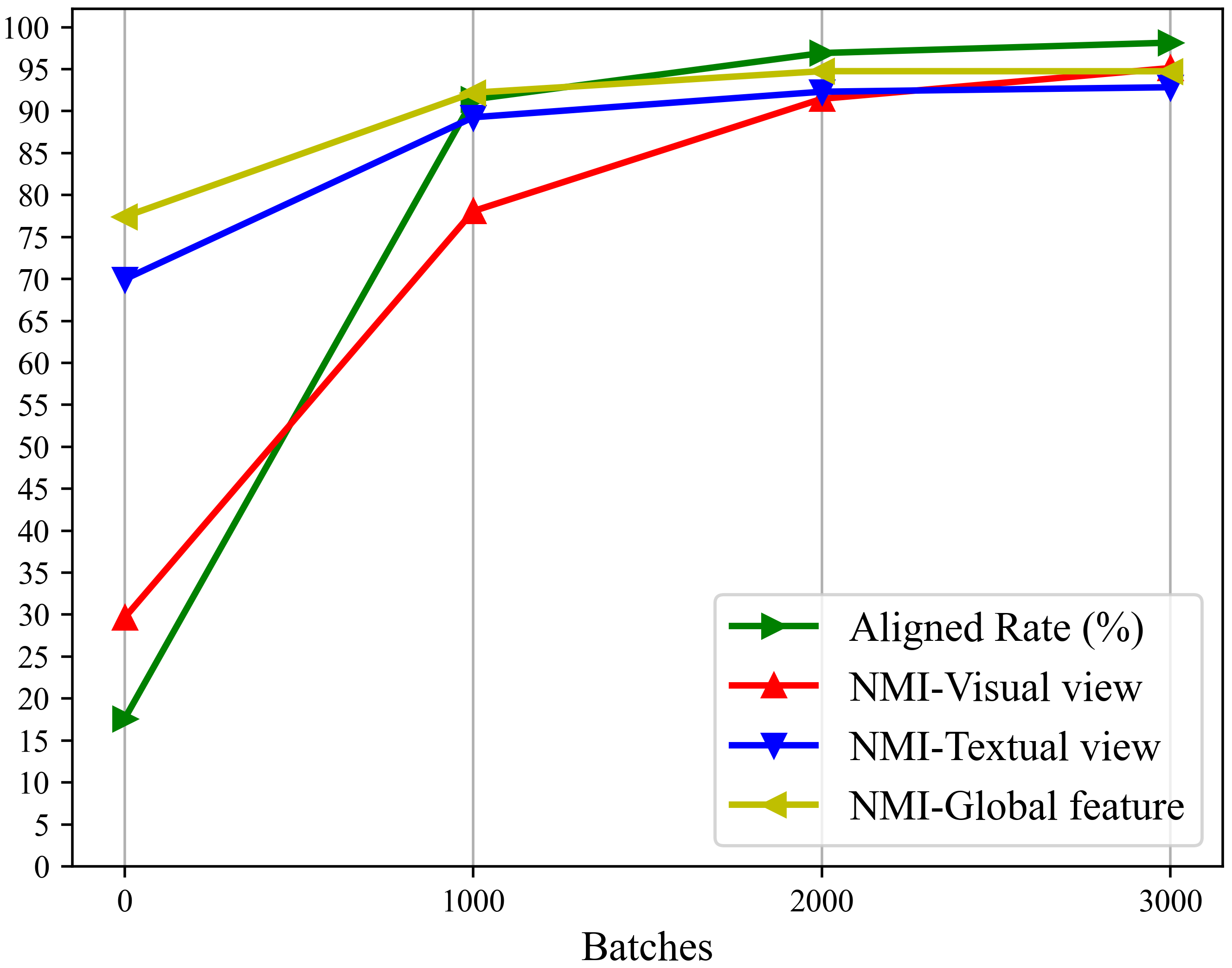}}
\caption{Aligned Rate and NMI during multi-view discriminative feature learning.}
\label{fig:Aligned}
\end{figure}

\subsection{Ablation Study}\label{sec:Ablation}

\subsubsection{Improvements Compared to Individual Views}
We test the clustering performance of IDEC on each independent view as shown in Table \ref{tab:table2}.
There is a considerable gap of performance between the best view of IDEC and the worst view of IDEC, where the low clustering performance indicates their clustering structures are unclear.
Without label information, we are not sure which view's prediction is better. The best view of IDEC and the corresponding view of SDMVC are highlighted in boldface. We find that SDMVC improves clustering performance by about 15$\%$ on MNIST-USPS, 30$\%$ on Fashion-MV, 3$\%$ on BDGP, and 30$\%$ on Caltech101-20. Eventually, the clustering performance of SDMVC on each view (even the views with the worst clustering performance) is much better than that of IDEC. This indicates that: (1) Our method have overcome the negative effects caused by the low discriminative views with unclear clustering structures. (2) The complementary information hidden in multiple views for each other can be explored to improve clustering performance by our proposed discriminative feature learning framework.

\subsubsection{Multi-View Discriminative Feature Learning}
In this part, we investigate how does our method work.
As shown in Fig. \ref{fig:separation}, we visualize the learning process of individual views' embedded features and global features on BDGP via $t$-SNE~\cite{maaten2008visualizing}.
Individual views' centroids of all clusters are also plotted, which belong to the learnable parameters in the clustering layers, i.e., $\{\{\pmb{\mu}_j^v\}{_{j=1}^K}\}{_{v=1}^V}$.
For example, when the pre-training of autoencoders is finished, two views' embedded features of BDGP are shown in Figs. \ref{fig:separation}(a) and (e), respectively.
The discriminative degree of features is low and the cluster centroids can not reflect the true clustering structures, corresponding to the low NMI values and low Aligned Rate as shown in Fig. \ref{fig:Aligned} (when \#Batches = 0).
By constructing the global features as shown in Fig. \ref{fig:separation}(i), SDMVC is able to explore the global discriminative information contained in the two views and thus the clustering performance of the global features is improved.
Subsequently, SDMVC builds a target distribution to learn more discriminative features as shown in Figs. \ref{fig:separation}(b) and (f), as well as discover the discriminative cluster patterns, which in turn are utilized to update the target distribution.
As a result, in the following feature learning process, the clustering structures of embedded features become clearer and clearer while their centroids are gradually separated, corresponding to the increasing NMI values and Aligned Rates.

The above observations verify the mechanism of SDMVC to boost clustering performance, i.e., by performing the proposed multi-view discriminative feature learning, SDMVC can globally utilize the discriminative and complementary information while overcoming the negative impact on clustering caused by some views with unclear clustering structures.

\begin{table*}[!t]
\centering
\renewcommand\tabcolsep{7.0pt}
\begin{threeparttable}
    \begin{tabular}{r|ccc|ccc|ccc|ccc}
    \toprule
    &\multicolumn{3}{c|}{MNIST-USPS} &\multicolumn{3}{c|}{Fashion-MV} &\multicolumn{3}{c|}{BDGP} &\multicolumn{3}{c}{Caltech101-20}\cr
    \hline
    Variants & ACC & NMI & ARI & ACC & NMI & ARI & ACC & NMI & ARI & ACC & NMI & ARI\\
    \hline
    SDMVC (w/o UTD) & 0.4742 & 0.5295 & 0.3407 & 0.4348 & 0.4587 & 0.3145 & 0.4624 & 0.2984 & 0.2483 & 0.2552 & 0.2912 & 0.1163 \\
    SDMVC (w/o SSM) & 0.8542 & 0.8873 & 0.8034 & 0.7163 & 0.7672 & 0.6831 & 0.9391 & 0.9051 & 0.9068 & 0.4639 & 0.6547 & 0.3825 \\
    SDMVC           & 0.9982 & 0.9947 & 0.9960 & 0.8626 & 0.9215 & 0.8405 & 0.9835 & 0.9466 & 0.9551 & 0.7158 & 0.7176 & 0.7265 \\
    \bottomrule
    \end{tabular}
    \caption{Ablation studies on the Unified Target Distribution (UTD) and the Self-Supervised Manner (SSM).
    }
    \label{tab:table3}
\end{threeparttable}
\end{table*}

\subsubsection{Unified Target Distribution (UTD) and Self-Supervised Manner (SSM)}

Table \ref{tab:table3} lists the results on two variants of SDMVC.
We could have the observations as follows:
(1) ``SDMVC (w/o UTD)'' does not yield satisfied results, which performs consistent clustering on multiple views' predictions by Eq. (\ref{eq:s}) without the proposed unified target distribution (UTD). Instead, the unified target distribution is used in SDMVC to learn consistent predictions. By averaging multiple views' predictions, SDMVC obtains the definite prediction, whose clustering performance is comparable and even better than that of the best view of SDMVC.
(2) The improvement of ``SDMVC (w/o SSM)'' over IDEC is also limited, which performs the feature learning without the proposed view-independent self-supervised manner (SSM). In SDMVC, however, the learning of each view’s discriminative features and cluster centroids is independent for other views, and the concatenated low-dimensional embedded features are only used to obtain the pseudo-label information. This framework preserves the diversity of multiple views while mining their complementary information, and thus greatly improves all views' clustering performance. Accordingly, the two variants of SDMVC validate that its different parts have necessary contributions.

\subsection{Model Analysis}
\begin{figure}[!t]
\centering
\subfigure{\includegraphics[height=2.1in]{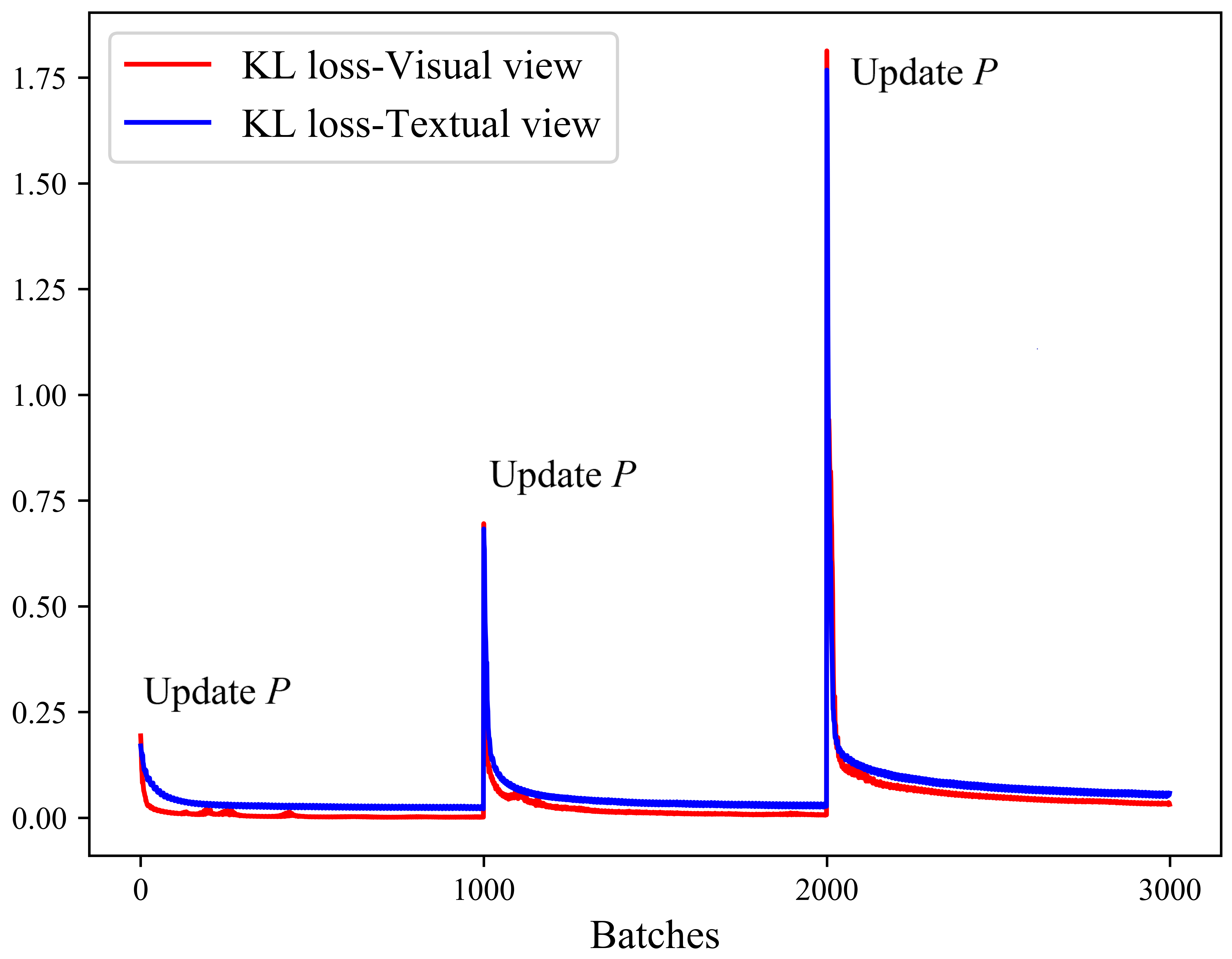}}
\caption{Clustering (KL divergence) loss.}
\label{fig:Convergence}
\end{figure}

\begin{figure}[!t]
\centering
\subfigure{\includegraphics[height=2.1in]{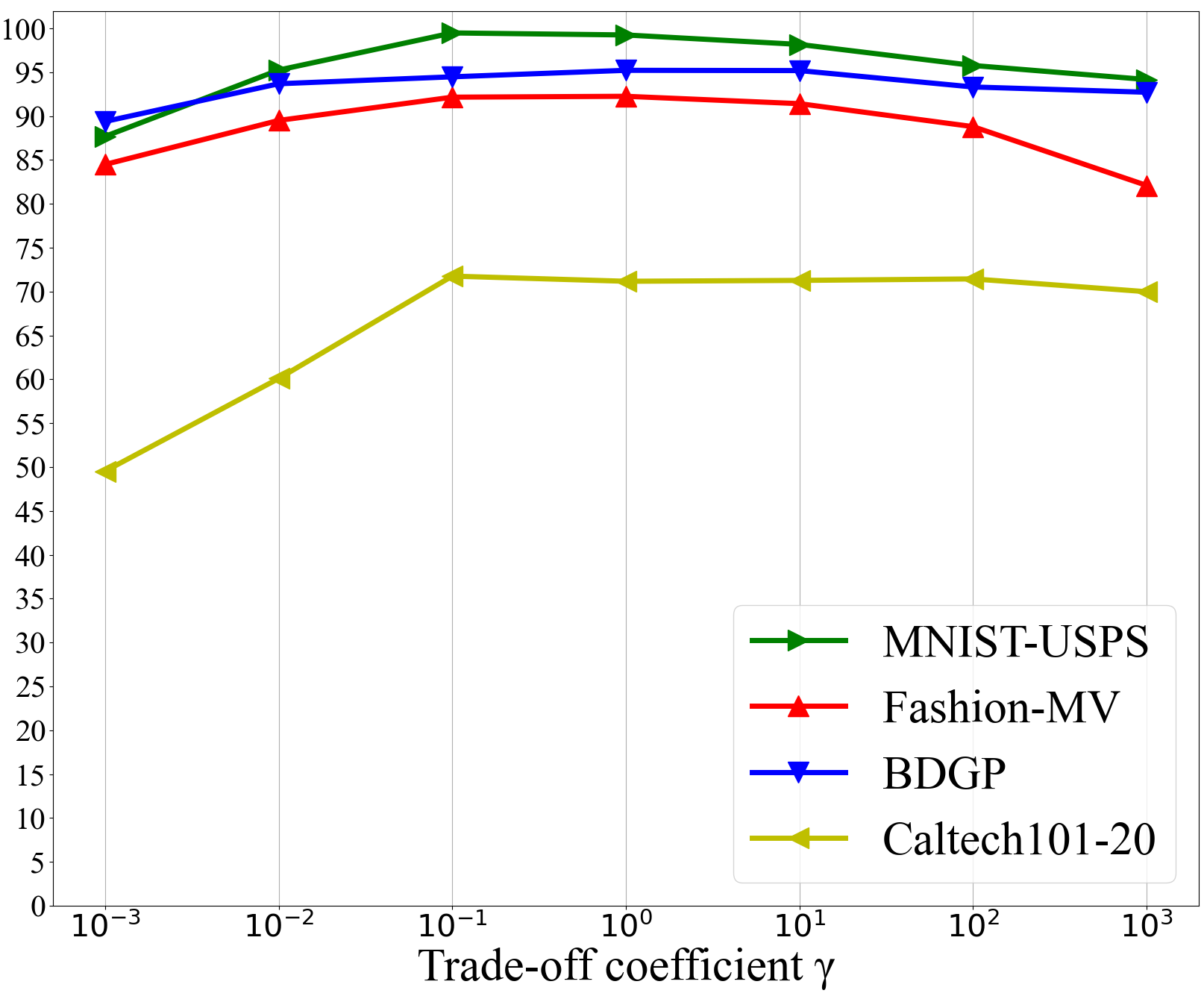}}
\caption{$\gamma$ $vs.$ NMI.}
\label{fig:parameters}
\end{figure}

\begin{figure}[!t]
\centering
\subfigure{\includegraphics[height=2.1in]{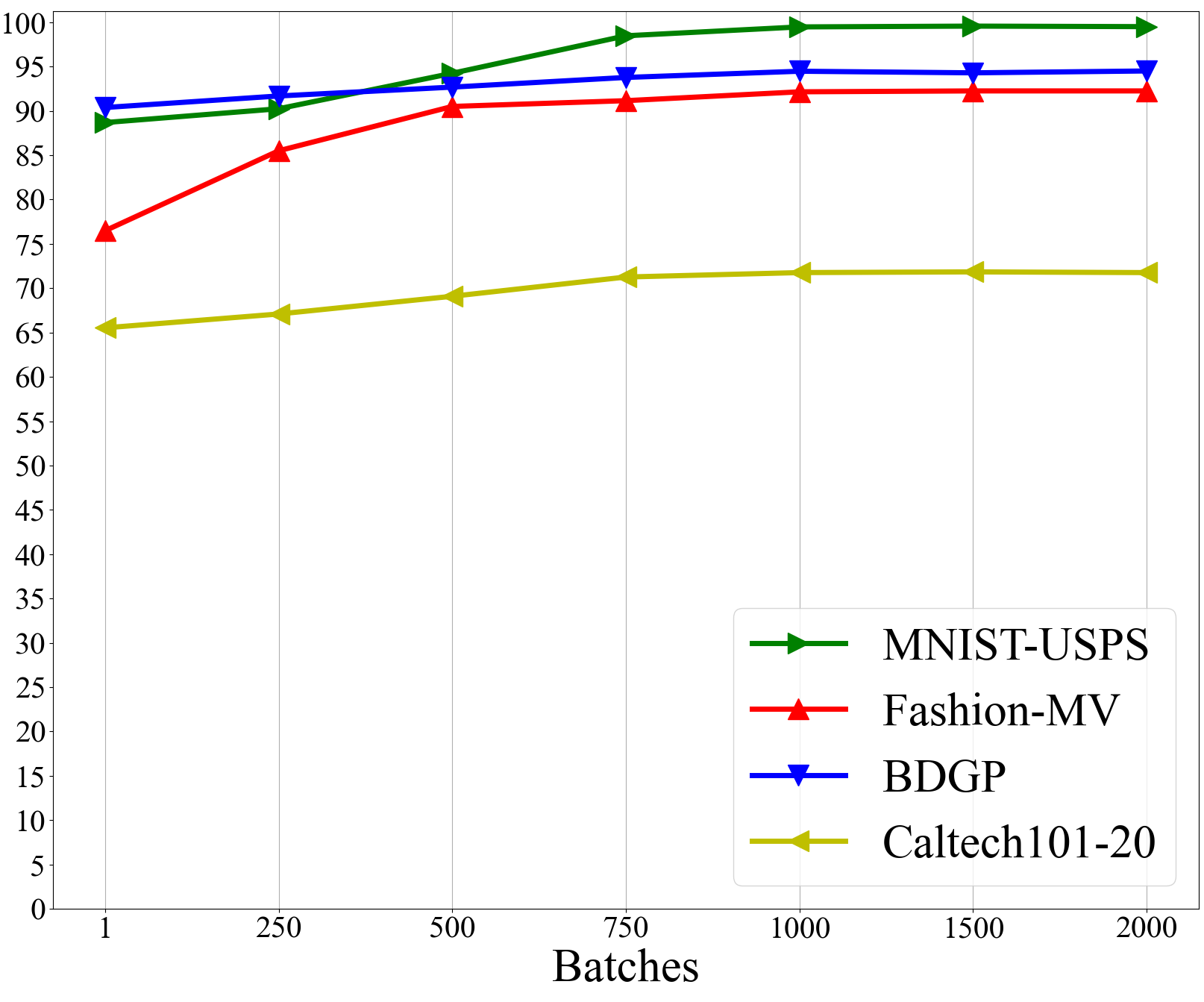}}
\caption{Batches $vs.$ NMI.}
\label{fig:Batches}
\end{figure}

\subsubsection{Convergence Analysis}
In this part, we conduct the model convergence analysis. In Fig. \ref{fig:Aligned}, we could find that Aligned Rate and clustering performance are positively correlated. In Fig. \ref{fig:Convergence}, each time the target distribution $P$ is updated, the newly generated and enhanced pseudo soft assignments have stronger discriminability. So, the clustering loss (i.e., KL divergence) values suddenly increase. Then, the soft cluster assignments of all views are trained to be consistent due to the unified target distribution, which improves the Aligned Rate of the multiple views' predictions.

It is not difficult to find that both the objectives of Eq. (\ref{eq:Lr}) and Eq. (\ref{eq:Lc}) are convex functions. Consequently, our algorithm has good convergence property after each update of the target distribution as shown in Fig. \ref{fig:Convergence}.
Since the Aligned Rate of our method is calculated in an unsupervised manner, we advise stopping training by taking a high value of Aligned Rate, such as $90\%$, to guarantee the multi-view clustering consistency.

\subsubsection{Hyper-parameter Analysis}
The setting of the trade-off coefficient $\gamma$ follows a common trade-off strategy between clustering and reconstruction.
Concretely, the clustering loss will be dominant when $\gamma$ is too large, and thus the feature diversity of multiple views may be degenerative. When $\gamma$ is too small, however, the clustering loss is unable to make the model learn meaningful clusters.
Fig. \ref{fig:parameters} reports the clustering results on the four datasets with varied trade-off coefficients (i.e., $\gamma$ varies from $10^{-3}$ to $10^{3}$). For SDMVC, the trade-off coefficient between the clustering loss and the reconstruction loss is robust in [$10^{-1}$, $10^{1}$]. Without loss of generality, we let $\gamma=0.1$ for all datasets in the experiments.

Additionally, we investigate the appropriate training batches to update the unified target distribution $P$. As shown in Fig. \ref{fig:Batches}, the performance is poor when the batches setting is too small. This is because the frequent update of target distribution $P$ make the model unable to utilize well the feature complementarity hidden in $P$. SDMVC has steady performance when the batches setting is large. For all datasets, accordingly, the target distribution is updated after fine-tuning every 1000 batches.

\section{Conclusion}
\label{sec:Conclusions}
In this paper, we have proposed a novel self-supervised discriminative feature learning framework for deep multi-view clustering (SDMVC).
Different from existing MVC methods, SDMVC can overcome the negative impact on clustering performance caused by certain views with unclear clustering structures.
In a self-supervised manner, it utilizes the discriminative information from a global perspective to establish a unified target distribution, which is used to learn more discriminative features and consistent predictions of multiple views.
We theoretically analyzed the generalization bound of our method.
Experiments on different types of multi-view datasets demonstrate that the proposed method has superior performance over the recent state-of-the-art multi-view clustering methods.
In addition, SDMVC has the complexity that is linear to the data size and thus its feature learning paradigm has the potential for other applications, such as semi-supervised clustering and classification.
Notably, we found that the stability of autoencoders' unsupervised pre-training might be a bottleneck for autoencoder-based multi-view clustering methods.
Therefore, our future work will focus on the robustness of deep multi-view feature learning and clustering.


%



\ifCLASSOPTIONcaptionsoff
  \newpage
\fi




%

\bibliographystyle{unsrt}
\bibliography{main}

\newpage
\onecolumn

\section*{Appendix}
\noindent\textcolor{black}{
\textbf{Theorem 1.}
\emph{
Let $\mathcal{L}(h)$ be the expectation of $\hat{\mathcal{L}}(h)$. Suppose that for any $\bm x \in \mathcal{X}$ and $h\in \mathcal{H}$, there exists $M  < \infty$ such that $\Vert{\bm x}^v\Vert, \Vert h_{\bm x}({\bm x})\Vert, \Vert h_{\bm z}({\bm x}) \Vert, \Vert {h_{\bm \mu}} \Vert \in [0, M] $ hold. With probability $1-\delta$ for any $h \in \mathcal{H}$, the following inequality holds
\begin{equation}\small
    \hat{\mathcal{L}}(h) \leq \mathcal{L}(h) + \frac{c_1}{\sqrt{N}} + c_2\sqrt{\frac{\log \frac{1}{\delta}}{2N}},
\end{equation}
where $c_1$ and $c_2$ are constants that depend on $K$, $M$ and $\gamma$.}}

\noindent\textcolor{black}{
\emph{Proof.}
To simplify the proof, we discuss the case that the target distribution is enhanced by an element-wise square without being divided by a factor. For the given sample set $S=\{{\bm x}_1,\cdots, {\bm x}_N \}$, let ${S}'$ be the sample set that differs from $S$ by only one sample point $\bar{\bm x}_r$. The empirical risk of the function $h$ on ${S}'$ is denoted as $\hat{\mathcal{L}}'(h)$. We have
\begin{equation*}\small
\begin{aligned}
    & \left\vert \mathop{\rm sup}_{h\in \mathcal{H}} \left\vert \hat{\mathcal{L}}(h) -\mathcal{L}(h) \right\vert - \mathop{\rm sup}_{h\in \mathcal{H}} \left\vert \hat{\mathcal{L}}'(h) -\mathcal{L}(h) \right\vert \right\vert \\
    \leq & \mathop{\rm sup}_{h\in \mathcal{H}} \left\vert \hat{\mathcal{L}}(h) - \hat{\mathcal{L}}'(h) \right\vert \\
    = & \mathop{\rm sup}_{h\in \mathcal{H}} \left\vert \frac{1}{N}\left( \Vert {\bm x}^v_r - h_{\bm x}({\bm x}^v_r) \Vert^2 - \Vert \bar{{\bm x}}^v_r - h_{\bm x}(\bar{\bm x}^v_r) \Vert^2 \right) \right\vert \\
    & + \gamma \mathop{\rm sup}_{h\in \mathcal{H}} \left\vert \frac{1}{N}\sum_{j=1}^K \left( p(h_{\bm z}(\bm{x}_r),{h_{\bm \mu}}_j) \log p(h_{\bm z}(\bm{x}_r),{h_{\bm \mu}}_j) - p(h_{\bm z}(\bar{\bm{x}}_r),{h_{\bm \mu}}_j) \log p(h_{\bm z}(\bar{\bm{x}}_r),{h_{\bm \mu}}_j) \right) \right\vert \\
    & + \gamma \mathop{\rm sup}_{h\in \mathcal{H}} \left\vert \frac{1}{N}\sum_{j=1}^K p(h_{\bm z}(\bm{x}_r),{h_{\bm \mu}}_j) \log \left(1+\Vert h_{\bm z}(\bm{x}^v_r)-{h_{\bm \mu}}^v_j\Vert^2\right) -  p(h_{\bm z}(\bar{\bm{x}}_r),{h_{\bm \mu}}_j) \log \left(1+\Vert h_{\bm z}(\bar{\bm{x}}^v_r)-{h_{\bm \mu}}^v_j\Vert^2\right)  \right\vert \\
    & + \gamma \mathop{\rm sup}_{h\in \mathcal{H}} \left\vert \frac{1}{N} \left[ \log \left( \sum_{j=1}^K {(1+\Vert h_{\bm z}(\bm{x}^v_r)-{h_{\bm \mu}}^v_j\Vert^2)}^{-1} \right) - \log \left( \sum_{j=1}^K {(1+\Vert h_{\bm z}(\bar{\bm{x}}^v_r)-{h_{\bm \mu}}^v_j\Vert^2)}^{-1} \right) \right] \right\vert \\
    \leq & \mathop{\rm sup}_{h\in \mathcal{H}} \left\vert \frac{1}{N}\left( \Vert {\bm x}^v_r - h_{\bm x}({\bm x}^v_r) \Vert^2 - \Vert \bar{{\bm x}}^v_r - h_{\bm x}(\bar{\bm x}^v_r) \Vert^2 \right) \right\vert \\
    & + \gamma \mathop{\rm sup}_{h\in \mathcal{H}} \left\vert \frac{1}{N}\sum_{j=1}^K  p(h_{\bm z}(\bm{x}_r),{h_{\bm \mu}}_j) \left[ \log p(h_{\bm z}(\bm{x}_r),{h_{\bm \mu}}_j) - \log p(h_{\bm z}(\bar{\bm{x}}_r),{h_{\bm \mu}}_j) \right]  \right\vert \\
    & + \gamma \mathop{\rm sup}_{h\in \mathcal{H}} \left\vert \frac{1}{N}\sum_{j=1}^K \left( \left[ p(h_{\bm z}(\bm{x}_r),{h_{\bm \mu}}_j) - p(h_{\bm z}(\bar{\bm{x}}_r),{h_{\bm \mu}}_j) \right] \log p(h_{\bm z}(\bar{\bm{x}}_r),{h_{\bm \mu}}_j) \right) \right\vert \\
    & + \gamma \mathop{\rm sup}_{h\in \mathcal{H}} \left\vert \frac{1}{N}\sum_{j=1}^K p(h_{\bm z}(\bm{x}_r),{h_{\bm \mu}}_j) \left[ \log \left(1+\Vert h_{\bm z}(\bm{x}^v_r)-{h_{\bm \mu}}^v_j\Vert^2\right) - \log \left(1+\Vert h_{\bm z}(\bar{\bm{x}}^v_r)-{h_{\bm \mu}}^v_j\Vert^2\right) \right] \right\vert \\
    & + \gamma \mathop{\rm sup}_{h\in \mathcal{H}} \left\vert \frac{1}{N}\sum_{j=1}^K \left[ p(h_{\bm z}(\bm{x}_r),{h_{\bm \mu}}_j) - p(h_{\bm z}(\bar{\bm{x}}_r),{h_{\bm \mu}}_j) \right] \log \left(1+\Vert h_{\bm z}(\bar{\bm{x}}^v_r)-{h_{\bm \mu}}^v_j\Vert^2\right) \right\vert \\
    & + \gamma \mathop{\rm sup}_{h\in \mathcal{H}} \left\vert \frac{1}{N} \left[ \log \left( \sum_{j=1}^K {(1+\Vert h_{\bm z}(\bm{x}^v_r)-{h_{\bm \mu}}^v_j\Vert^2)}^{-1} \right) - \log \left( \sum_{j=1}^K {(1+\Vert h_{\bm z}(\bar{\bm{x}}^v_r)-{h_{\bm \mu}}^v_j\Vert^2)}^{-1} \right) \right] \right\vert \\
    \leq &  \mathop{\rm sup}_{h\in \mathcal{H}} \left\vert \frac{1}{N}\left( \Vert {\bm x}^v_r - h_{\bm x}({\bm x}^v_r) \Vert^2 - \Vert \bar{{\bm x}}^v_r - h_{\bm x}(\bar{\bm x}^v_r) \Vert^2 \right) \right\vert \\
    & + \gamma \mathop{\rm sup}_{h\in \mathcal{H}}  \frac{1}{N}\sum_{j=1}^K p(h_{\bm z}(\bm{x}_r),{h_{\bm \mu}}_j) {\xi}_c \left\vert p(h_{\bm z}(\bm{x}_r),{h_{\bm \mu}}_j) - p(h_{\bm z}(\bar{\bm{x}}_r),{h_{\bm \mu}}_j) \right\vert  + \gamma \mathop{\rm sup}_{h\in \mathcal{H}}  \frac{1}{N}\sum_{j=1}^K \left\vert \log p(h_{\bm z}(\bar{\bm{x}}_r),{h_{\bm \mu}}_j) \right\vert \\
    & + \gamma \mathop{\rm sup}_{h\in \mathcal{H}} \frac{1}{N}\sum_{j=1}^K p(h_{\bm z}(\bm{x}_r),{h_{\bm \mu}}_j) {\xi}_a \left\vert \left(1+\Vert h_{\bm z}(\bm{x}^v_r)-{h_{\bm \mu}}^v_j\Vert^2\right) - \left(1+\Vert h_{\bm z}(\bar{\bm{x}}^v_r)-{h_{\bm \mu}}^v_j\Vert^2\right)\right\vert \\
    & + \gamma \mathop{\rm sup}_{h\in \mathcal{H}}  \frac{1}{N}\sum_{j=1}^K \left\vert \log \left(1+\Vert h_{\bm z}(\bar{\bm{x}}^v_i)-{h_{\bm \mu}}^v_j\Vert^2\right) \right\vert \\
    & + \gamma \mathop{\rm sup}_{h\in \mathcal{H}} \frac{1}{N} \sum_{j=1}^K {\xi}_b \left\vert {(1+\Vert h_{\bm z}(\bm{x}^v_r)-{h_{\bm \mu}}^v_j\Vert^2)}^{-1} - {(1+\Vert h_{\bm z}(\bar{\bm{x}}^v_r)-{h_{\bm \mu}}^v_j\Vert^2)}^{-1} \right\vert,
\end{aligned}
\end{equation*}
where the last inequality is based on the Lagrange mean value theorem. Let $C:={\rm max}\{{\xi}_a,{\xi}_b,{\xi}_c\}$, we have
\begin{equation*}\small
\begin{aligned}
    & \left\vert \mathop{\rm sup}_{h\in \mathcal{H}} \left\vert \hat{\mathcal{L}}(h) -\mathcal{L}(h) \right\vert - \mathop{\rm sup}_{h\in \mathcal{H}} \left\vert \hat{\mathcal{L}}'(h) -\mathcal{L}(h) \right\vert \right\vert \\
    \leq & \mathop{\rm sup}_{h\in \mathcal{H}} \frac{1}{N}\left\vert \left( \Vert {\bm x}^v_r - h_{\bm x}({\bm x}^v_r) \Vert^2 - \Vert \bar{{\bm x}}^v_r - h_{\bm x}(\bar{\bm x}^v_r) \Vert^2 \right) \right\vert \\
    & + \gamma \mathop{\rm sup}_{h\in \mathcal{H}} \frac{1}{N}\sum_{j=1}^K p(h_{\bm z}(\bm{x}_r),{h_{\bm \mu}}_j) {\xi}_a \left\vert \Vert h_{\bm z}(\bm{x}^v_r)-{h_{\bm \mu}}^v_j\Vert^2 - \Vert h_{\bm z}(\bar{\bm{x}}^v_r)-{h_{\bm \mu}}^v_j\Vert^2 \right\vert \\
    & + \gamma \mathop{\rm sup}_{h\in \mathcal{H}} \frac{1}{N} \sum_{j=1}^K {\xi}_b \frac{ \left\vert \Vert h_{\bm z}(\bar{\bm{x}}^v_r)-{h_{\bm \mu}}^v_j\Vert^2 - \Vert h_{\bm z}(\bm{x}^v_r)-{h_{\bm \mu}}^v_j\Vert^2 \right\vert }{(1+\Vert h_{\bm z}(\bm{x}^v_r)-{h_{\bm \mu}}^v_j\Vert^2)(1+\Vert h_{\bm z}(\bar{\bm{x}}^v_r)-{h_{\bm \mu}}^v_j\Vert^2)} \\
    & + \frac{\gamma K(\log (1+4M^2) + {\xi}_c + \tilde{c})}{N} \\
    \leq & \mathop{\rm sup}_{h\in \mathcal{H}} \frac{1}{N}\left( 2M^2 + 2\Vert {\bm x}^v_r\Vert \Vert h_{\bm x}({\bm x}^v_r) \Vert + 2\Vert \bar{\bm x}^v_r\Vert \Vert h_{\bm x}(\bar{\bm x}^v_r) \Vert  \right) \\
    & + \gamma \mathop{\rm sup}_{h\in \mathcal{H}} \frac{1}{N}\sum_{j=1}^K p(h_{\bm z}(\bm{x}_r),{h_{\bm \mu}}_j) {\xi}_a \left( 2\Vert h_{\bm z}(\bm{x}^v_r)\Vert \Vert{h_{\bm \mu}}^v_j\Vert + 2\Vert h_{\bm z}(\bar{\bm{x}}^v_r) \Vert \Vert {h_{\bm \mu}}^v_j \Vert + M^2 \right) \\
    & + \gamma \mathop{\rm sup}_{h\in \mathcal{H}} \frac{1}{N} \sum_{j=1}^K {\xi}_b \left( 2\Vert h_{\bm z}(\bm{x}^v_r)\Vert \Vert{h_{\bm \mu}}^v_j\Vert + 2\Vert h_{\bm z}(\bar{\bm{x}}^v_r) \Vert \Vert {h_{\bm \mu}}^v_j \Vert + M^2 \right) \\
    & + \frac{\gamma K(\log (1+4M^2) + {\xi}_c + \tilde{c})}{N} \\
    \leq & \frac{5\gamma KM^2({\xi}_a+{\xi}_b)+6M^2}{N} + \frac{\gamma K(\log (1+4M^2) + {\xi}_c + \tilde{c})}{N} \\
    \leq & \frac{10\gamma M^2CK+\gamma K(\log (1+4M^2)+C+\tilde{c})+6M^2}{N}:=\frac{c_{K,M}}{N},
\end{aligned}
\end{equation*}
where $\tilde{c}=\log \frac{(4M^2+K)^2+(K-1)(4M^2+1)^2[1+(K-1)(4M^2+1)]^2}{(4M^2+K)^2}$, and $c_{K,M,\gamma}= 10\gamma M^2CK+\gamma K(\log (1+4M^2)+C+\tilde{c})+6M^2$ denotes the constant that dependents on $K$, $M$ and $\gamma$. Next, we bound the expectation term $\mathbb{E}_S \mathop{\rm sup}_{h \in \mathcal{H}} \left\vert \mathcal{L}(h) - \hat{\mathcal{L}}(h) \right\vert$. Let ${\sigma}_1,\ldots,{\sigma}_N$ be i.i.d. Rademacher random variables. First we have
\begin{equation}\small\label{expect_term}
\begin{aligned}
    & \mathbb{E}_S \mathop{\rm sup}_{h \in \mathcal{H}} \left\vert \mathcal{L}(h) - \hat{\mathcal{L}}(h) \right\vert \\
    \leq & \mathbb{E}_{S,\bar{S}} \mathop{\rm sup}_{h \in \mathcal{H}} \left\vert \frac{1}{N}\sum_{i=1}^N \left( \Vert {\bm x}^v_i - h_{\bm x}({\bm x}^v_i) \Vert^2 - \Vert \bar{{\bm x}}^v_i - h_{\bm x}(\bar{\bm x}^v_i) \Vert^2 \right) \right\vert \\
    & + \gamma \mathbb{E}_{S,\bar{S}} \mathop{\rm sup}_{h\in \mathcal{H}} \left\vert \frac{1}{N}\sum_{i=1}^N\sum_{j=1}^K  \left( p(h_{\bm z}(\bm{x}_i),{h_{\bm \mu}}_j)  \log p(h_{\bm z}(\bm{x}_i),{h_{\bm \mu}}_j) - p(h_{\bm z}(\bar{\bm{x}}_i),{h_{\bm \mu}}_j) \log p(h_{\bm z}(\bar{\bm{x}}_i),{h_{\bm \mu}}_j) \right)  \right\vert \\
    & + \gamma \mathbb{E}_{S,\bar{S}} \mathop{\rm sup}_{h\in \mathcal{H}} \left\vert \frac{1}{N}\sum_{i=1}^N\sum_{j=1}^K \left( p(h_{\bm z}(\bm{x}_i),{h_{\bm \mu}}_j) \log \left(1+\Vert h_{\bm z}(\bm{x}^v_i)-{h_{\bm \mu}}^v_j\Vert^2\right) - p(h_{\bm z}(\bar{\bm{x}}_i),{h_{\bm \mu}}_j) \log \left(1+\Vert h_{\bm z}(\bar{\bm{x}}^v_i)-{h_{\bm \mu}}^v_j\Vert^2\right) \right) \right\vert \\
    & + \gamma \mathbb{E}_{S,\bar{S}} \mathop{\rm sup}_{h \in \mathcal{H}} \left\vert \frac{1}{N}\sum_{i=1}^N \left[ \log \left( \sum_{j=1}^K {(1+\Vert h_{\bm z}(\bm{x}^v_i)-{h_{\bm \mu}}^v_j\Vert^2)}^{-1} \right) - \log \left( \sum_{j=1}^K {(1+\Vert h_{\bm z}(\bar{\bm{x}}^v_i)-{h_{\bm \mu}}^v_j\Vert^2)}^{-1} \right)\right] \right\vert \\
    = & 2\mathbb{E}_{S,\bm \sigma} \mathop{\rm sup}_{h \in \mathcal{H}} \left\vert \frac{1}{N}\sum_{i=1}^N {\sigma}_i \Vert {\bm x}^v_r - h_{\bm x}({\bm x}^v_r) \Vert^2 \right\vert + 2\gamma \mathbb{E}_{S,\bm \sigma} \mathop{\rm sup}_{h\in \mathcal{H}} \left\vert \frac{1}{N}\sum_{i=1}^N\sum_{j=1}^K  {\sigma}_i p(h_{\bm z}(\bm{x}_i),{h_{\bm \mu}}_j)  \log p(h_{\bm z}(\bm{x}_i),{h_{\bm \mu}}_j) \right\vert \\
    & + 2\gamma \mathbb{E}_{S,\bm \sigma} \mathop{\rm sup}_{h \in \mathcal{H}} \left\vert \frac{1}{N}\sum_{i=1}^N \sum_{j=1}^K {\sigma}_i p(h_{\bm z}(\bm{x}_i),{h_{\bm \mu}}_j)  \log \left(1+\Vert h_{\bm z}(\bm{x}^v_i)-{h_{\bm \mu}}^v_j\Vert^2\right) \right\vert \\
    & + 2\gamma \mathbb{E}_{S,\bm \sigma} \mathop{\rm sup}_{h \in \mathcal{H}} \left\vert \frac{1}{N}\sum_{i=1}^N {\sigma}_i \log \left( \sum_{j=1}^K {(1+\Vert h_{\bm z}(\bm{x}^v_i)-{h_{\bm \mu}}^v_j\Vert^2)}^{-1} \right) \right\vert \\
    \leq & 2\mathbb{E}_{S,\bm \sigma} \mathop{\rm sup}_{h \in \mathcal{H}} \left\vert \frac{1}{N}\sum_{i=1}^N {\sigma}_i \Vert {\bm x}^v_r - h_{\bm x}({\bm x}^v_r) \Vert^2 \right\vert + 2\gamma K \mathop{\rm max}_{j} \mathbb{E}_{S,\bm \sigma} \mathop{\rm sup}_{h\in \mathcal{H}} \left\vert \frac{1}{N}\sum_{i=1}^N  {\sigma}_i p(h_{\bm z}(\bm{x}_i),{h_{\bm \mu}}_j)  \log p(h_{\bm z}(\bm{x}_i),{h_{\bm \mu}}_j) \right\vert \\
    & + 2\gamma K\mathop{\rm max}_{j} \mathbb{E}_{S,\bm \sigma} \mathop{\rm sup}_{h \in \mathcal{H}} \left\vert \frac{1}{N}\sum_{i=1}^N {\sigma}_i p(h_{\bm z}(\bm{x}_i),{h_{\bm \mu}}_j)  \log \left(1+\Vert h_{\bm z}(\bm{x}^v_i)-{h_{\bm \mu}}^v_j\Vert^2\right) \right\vert \\
    & + 2\gamma \mathbb{E}_{S,\bm \sigma} \mathop{\rm sup}_{h \in \mathcal{H}} \left\vert \frac{1}{N}\sum_{i=1}^N {\sigma}_i \log \left( \sum_{j=1}^K {(1+\Vert h_{\bm z}(\bm{x}^v_i)-{h_{\bm \mu}}^v_j\Vert^2)}^{-1} \right) \right\vert.
\end{aligned}
\end{equation}
For the first term, according to the Khintchine-Kahane inequality \cite{latala1994}, we have
\begin{equation}\small\label{expect_term_1}
\begin{aligned}
    & 2\mathbb{E}_{S,\bm \sigma} \mathop{\rm sup}_{h \in \mathcal{H}} \left\vert \frac{1}{N}\sum_{i=1}^N {\sigma}_i \Vert {\bm x}^v_r - h_{\bm x}({\bm x}^v_r) \Vert^2 \right\vert \leq 2\mathbb{E}_{S,\bm \sigma} \mathop{\rm sup}_{h \in \mathcal{H}} \frac{1}{N}\left( \sum_{i=1}^N [\Vert {\bm x}^v_r - h_{\bm x}({\bm x}^v_r) \Vert^2]^2 \right)^{\frac{1}{2}} \leq \frac{8M^2}{\sqrt{N}}.
\end{aligned}
\end{equation}
Similarly, we have
\begin{equation}\small\label{expect_term_2}
\begin{aligned}
    & 2\gamma K \mathop{\rm max}_{j} \mathbb{E}_{S,\bm \sigma} \mathop{\rm sup}_{h\in \mathcal{H}} \left\vert \frac{1}{N}\sum_{i=1}^N  p(h_{\bm z}(\bm{x}_i),{h_{\bm \mu}}_j)  \log p(h_{\bm z}(\bm{x}_i),{h_{\bm \mu}}_j) \right\vert \leq \frac{2\gamma K\tilde{c}}{\sqrt{N}} \\
    & 2\gamma K\mathop{\rm max}_{j} \mathbb{E}_{S,\bm \sigma} \mathop{\rm sup}_{h \in \mathcal{H}} \left\vert \frac{1}{N}\sum_{i=1}^N p(h_{\bm z}(\bm{x}_i),{h_{\bm \mu}}_j) {\sigma}_i \log \left(1+\Vert h_{\bm z}(\bm{x}^v_i)-{h_{\bm \mu}}^v_j\Vert^2\right) \right\vert \leq \frac{8\gamma KM^2}{\sqrt{N}}.
\end{aligned}
\end{equation}
For the last term, note that the following inequality holds
\begin{equation*}\small
    \log \frac{K}{4M^2+1} \leq {\log} \left( \sum_{j=1}^K {(1+\Vert h_{\bm z}(\bm{x}^v_i)-{h_{\bm \mu}}^v_j\Vert^2)}^{-1} \right) \leq {\log} K.
\end{equation*}
And thus
\begin{equation}\small\label{expect_term_3}
\begin{aligned}
    & 2\gamma \mathbb{E}_{S,\bm \sigma} \mathop{\rm sup}_{h \in \mathcal{H}} \left\vert \frac{1}{N}\sum_{i=1}^N {\sigma}_i \log \left( \sum_{j=1}^K {(1+\Vert h_{\bm z}(\bm{x}^v_i)-{h_{\bm \mu}}^v_j\Vert^2)}^{-1} \right) \right\vert \leq \frac{2\gamma[{\log} K + \log (4M^2+1)]}{\sqrt{N}}.
\end{aligned}
\end{equation}
Combining (\ref{expect_term}), (\ref{expect_term_1}), (\ref{expect_term_2}) and (\ref{expect_term_3}), we have
\begin{equation*}\small
    \mathbb{E}_S \mathop{\rm sup}_{h \in \mathcal{H}} \left\vert \mathcal{L}(f) - \hat{\mathcal{L}}(f) \right\vert \leq \frac{8M^2(\gamma K+1) + 2\gamma K\tilde{c} + 2\gamma[{\log} K +\log(4M^2+1)]}{\sqrt{N}}.
\end{equation*}
According to the McDiarmid inequality \cite{mohri2018foundations}, we conclude that with probability $1-\delta$ for any $h \in \mathcal{H}$ the following inequality holds
\begin{equation*}\small
    \hat{\mathcal{L}}(h) \leq \mathcal{L}(h) + \frac{8M^2(\gamma K+1) + 2\gamma K\tilde{c} + 2\gamma[{\log} K +\log(4M^2+1)]}{\sqrt{N}} + c_{K,M,\gamma}\sqrt{\frac{\log \frac{1}{\delta}}{2N}}.
\end{equation*}
Let $c_1:=8M^2(\gamma K+1) + 2\gamma K\tilde{c} + 2\gamma[{\log} K +\log(4M^2+1)]$ and $c_2:=c_{K,M,\gamma}$, the proof is finished.}

\end{document}